%% file: main.tex
\documentclass[letterpaper, 10 pt, conference]{ieeeconf}  
\IEEEoverridecommandlockouts                              
\usepackage{times}
\usepackage[pdftex]{graphicx}
\input{header.tex}

\newif\ifarxiv
\arxivtrue
\makeatletter
\IEEEtriggercmd{\reset@font\normalfont\fontsize{7.0pt}{7.2pt}\selectfont}
\makeatother
\IEEEtriggeratref{1}

\makeatletter
\newcommand\fs@ruled@notop{\def\@fs@cfont{\bfseries}\let\@fs@capt\floatc@ruled
  \def\@fs@pre{}%
  \def\@fs@post{\kern2pt\hrule\relax}%
  \def\@fs@mid{\kern2pt\hrule\kern2pt}%
  \let\@fs@iftopcapt\iftrue}
\renewcommand\fst@algorithm{\fs@ruled@notop}
\makeatother
\allowdisplaybreaks

\title{\large\bf Implicit Integration for Articulated Bodies with Contact\\
via the Nonconvex Maximal Dissipation Principle \vspace{-15px}}
\author{Zherong Pan and Kris Hauser$^\dagger$  \\
\vspace{-200px}
\thanks{$^\dagger$ Zherong Pan and Kris Hauser are with the Department of Computer Science, University of Illinois at Urbana-Champaign. {\tt\small \{zherong,kkhauser\}@illinois.edu}}}
        
\begin{document}
\maketitle
\thispagestyle{empty}
\pagestyle{empty}

\newcommand{\KRIS}[1]{\textcolor{red}{#1}}
\begin{abstract}
We present non-convex maximal dissipation principle (NMDP), a time integration scheme for articulated bodies with simultaneous contacts. Our scheme resolves contact forces via the maximal dissipation principle (MDP). Prior MDP solvers compute contact forces via convex programming by assuming linearized dynamics integrated using the forward multistep scheme. Instead, we consider the coupled system of nonlinear Newton-Euler dynamics and MDP, which is time-integrated using the backward integration scheme. We show that the coupled system of equations can be solved efficiently using the projected gradient method with guaranteed convergence. We evaluate our method by predicting several locomotion trajectories for a quadruped robot. The results show that our NMDP scheme has several desirable properties including: (1) generalization to novel contact models; (2) superior stability under large timestep sizes; (3) consistent trajectory generation under varying timestep sizes.
\end{abstract}
\input{introduction.tex}
\input{related.tex}
\input{problem.tex}
\input{method.tex}
\input{analysis.tex}
\input{results.tex}
\input{conclusion.tex}
\section{Acknowledgement}
This work is partially funded by NSF Grant \#1911087 and authors thank Mengchao Zhang for proofreading the paper.
\clearpage
\bibliographystyle{IEEEtranS}
\bibliography{references}
\ifarxiv
\input{appendix.tex}

\fi
\end{document}

%% file: header.tex
\usepackage{amsmath,amsfonts,amssymb}
\usepackage{tabu}
\usepackage{bbm}
\usepackage{prettyref}
\usepackage{mathrsfs}
\usepackage{graphicx}
\usepackage{wrapfig}
\usepackage{subfig}
\usepackage{MnSymbol}
\usepackage{multirow}
\usepackage{booktabs}
\usepackage{algorithm}
\usepackage[noend]{algpseudocode}
\usepackage[table]{xcolor}
\usepackage{cellspace}
\usepackage{mathtools}
\usepackage{cite}

\newtheorem{remark}{Remark}
\newtheorem{theorem}{Theorem}[section]
\newtheorem{assume}[theorem]{\TE{A}}
\newtheorem{lemma}[theorem]{\TE{Lemma}}

\newtheorem{corollary}[theorem]{\TE{Corollary}}

\algnewcommand{\LineComment}[1]{\State \(\triangleright\) #1}
\algdef{SE}[DOWHILE]{Do}{doWhile}{\algorithmicdo}[1]{\algorithmicwhile\ #1}
\newcommand*{\colorboxed}{}
\def\colorboxed#1#{%
  \colorboxedAux{#1}%
}
\newcommand*{\colorboxedAux}[3]{%
  \begingroup
    \colorlet{cb@saved}{.}%
    \color#1{#2}%
    \boxed{%
      \color{cb@saved}%
      #3%
    }%
  \endgroup
}

\newrefformat{Fig}{Figure~\ref{#1}}
\newrefformat{fig}{Figure~\ref{#1}}
\newrefformat{par}{Section~\ref{#1}}
\newrefformat{appen}{Appendix~\ref{#1}}
\newrefformat{ass}{A~\ref{#1}}
\newrefformat{sec}{Section~\ref{#1}}
\newrefformat{sub}{Section~\ref{#1}}
\newrefformat{table}{Table~\ref{#1}}
\newrefformat{alg}{Algorithm~\ref{#1}}
\newrefformat{Alg}{Algorithm~\ref{#1}}
\newrefformat{Def}{Definition~\ref{#1}}
\newrefformat{Cor}{Corollary~\ref{#1}}
\newrefformat{Thm}{Theorem~\ref{#1}}
\newrefformat{Lem}{Lemma~\ref{#1}}
\newrefformat{step}{Step~\ref{#1}}
\newrefformat{ln}{Line~\ref{#1}}
\newrefformat{eq}{Equation~\ref{#1}}
\newrefformat{pb}{Problem~\ref{#1}}
\newrefformat{it}{Item~\ref{#1}}
\newrefformat{te}{Term~\ref{#1}}
\def\Eqref Eq:#1:{\eqref{eq:#1}}
\newrefformat{Eq}{Equation~\Eqref#1:}

\newcommand{\E}[1]{\mathbf{#1}}
\newcommand{\TE}[1]{\textbf{#1}}

\newcommand{\FPPN}[2]{\nabla_{#2}{#1}}
\newcommand{\FPPSN}[2]{\nabla_{#2}^*{#1}}

\newcommand{\FPPTN}[2]{\nabla_{#2}^2{#1}}
\newcommand{\FPPTTTN}[2]{\nabla_{#2}^3{#1}}

\newcommand{\FPPTTT}[2]{\frac{\partial^3{#1}}{\partial{#2}^3}}

\newcommand{\THREE}[3]{\left(\setlength{\arraycolsep}{1pt}\begin{array}{ccc}{#1}, & {#2}, & {#3}\end{array}\right)}

\newcommand{\argmin}[1]{\underset{#1}{\E{argmin}}\;}
\newcommand{\argminP}[1]{\E{argmin}\;}

\newcommand{\argmaxP}[1]{\E{argmax}\;}
\newcommand{\ST}{\E{s.t.}\;}

\newcommand{\BECAUSE}[1]{\;\text{(#1)}}

%% file: introduction.tex
\section{\label{sec:intro}Introduction}
Articulated body simulation is an indispensable component of robot motion planning, optimal control, and reinforcement learning (RL). Their governing dynamic equations, i.e. the recursive Newton-Euler's equation \cite{featherstone2014rigid}, and discretization schemes have been studied for decades. However, efficient and accurate contact handling is still a challenging problem studied extensively by recent works \cite{preclik2018maximum,8794337}. To predict robot motions under simultaneous Coulomb frictional contacts, the two most widely-used formulations are the linear-complementary problem (LCP) \cite{stewart2000rigid} and the maximal dissipation principle (MDP) \cite{drumwright2010modeling}. From a computational perspective, LCP incurs an NP-hard problem while MDP identifies contact forces with the solution of a cheap-to-compute convex program. As reported by \cite{erez2015simulation}, MDP-based contact handler achieves the best stability and computational efficiency. Moreover, MDP can encode novel contact models as arbitrary convex wrench spaces, which enables learning contact models from data~\cite{Zhou2018,8794337}. 

The stability region of MDP is shown to be up to $\sim$10\,ms according to \cite{tassa2012synthesis}. Beyond the stability region, the predicted trajectory can either blow-up or drift significantly from the ground truth. Such small stability region not only increases computational cost but also induces problems of vanishing or exploding gradients \cite{liu2020differential}. In contact-implicit trajectory optimization \cite{mordatch2012discovery,8968194}, for example, the problem sizes grow linearly with the number of timesteps and the cost of a Newton-type method grows superlinearly as a result.
\input{teaserNew.tex}

We present  a non-convex MDP (NMDP) integrator that: (1) is stable under large timestep sizes; (2) generates consistent contact forces under the MDP formulation; and (3) generalizes to position-dependent contact models. Prior MDP solvers rely on linearized dynamic systems, so that the kinetic energy becomes a quadratic function of the contact forces which can be solved as a convex QP. However, the truncation error of linearization can grow arbitrarily with larger timestep sizes. Our NMDP solver eliminates the truncation error by formulating the {\em nonlinear} recursive Newton-Euler's equation and the wrench space as a function of the robot pose as a coupled system of nonlinear equations time-integrated using the backward-Euler scheme (\prettyref{fig:illusb}). The method can inherently account for novel contact models by using the convex shapes as feasible constraints in MDP, with nonlinear dependence on robot pose. To solve this coupled system we propose using the projected gradient method (PGM). We prove that PGM converges under sufficiently small timesteps and show that it empirically converges under large timesteps. An adaptive inner time-integration scheme guarantees that NMDP solves any (primary) timestep size in finite time.

We evaluate our method by predicting walking and jumping trajectories for the JPL Robosimian and Spider quadruped robot. The results show that NMDP has superior stability under larger timesteps as compared with conventional MDP solvers. In addition, the predicted walking speed and jumping height are more consistent under various timesteps.

%% file: teaserNew.tex
\begin{figure}[ht]
\centering
\includegraphics[width=.45\textwidth]{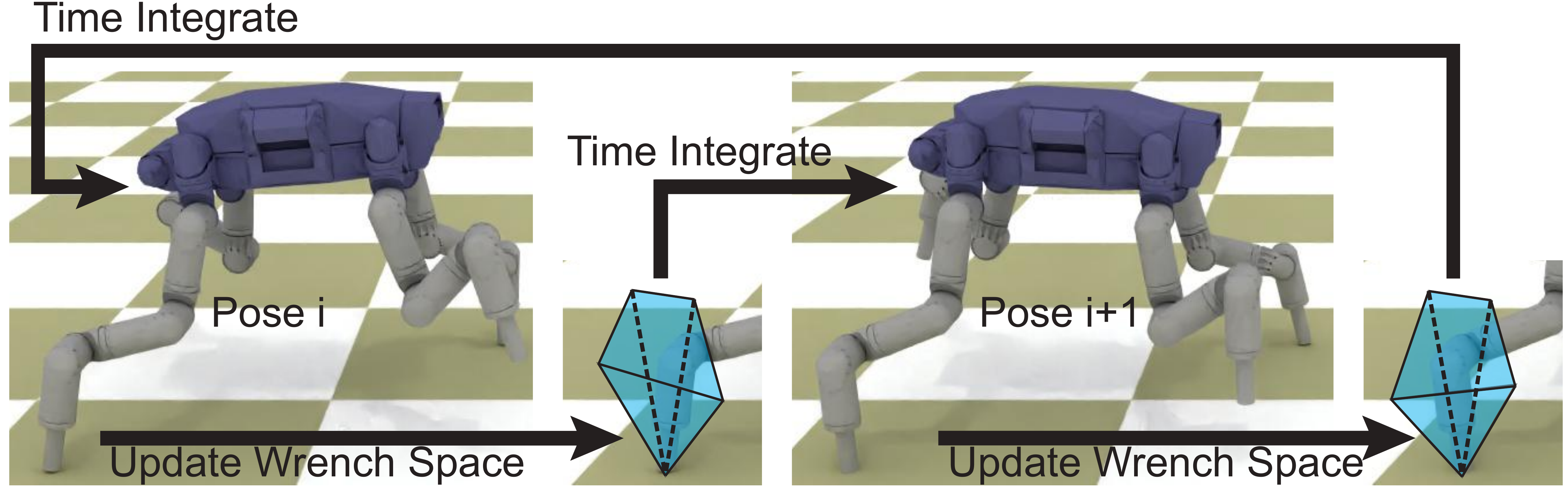}
\caption{\label{fig:illusb} \small{We consider the three components (robot pose $\theta$ update, convex wrench space $v_x(\theta)$ update, and contact force $w$ update) as a coupled system of nonlinear equations, which is solved using a novel projected gradient method with guaranteed convergence.}}
\vspace{-10px}
\end{figure}

%% file: related.tex
\section{\label{sec:Related}Related Work}
We review related work on articulated body dynamics, contact handling, and generalized contact models.

\TE{Articulated Body Dynamics:} Three classes of time discretization schemes have been independently developed for articulated bodies' equation-of-motion. First, variational integrators (VI) \cite{lee2020linear,marsden2001discrete} discretize the Lagrangian function and then derive the discrete Euler-Lagrange equation. VI preserves momentum and energy symmetry under large timestep sizes. Second, linear multi-step integrators \cite{butcher2008numerical} discretize the equivalent Newton-Euler's equation in the configuration space. These integrators are very efficient to evaluate using the Articulated Body Algorithm (ABA) \cite{featherstone2014rigid}. Third, high-order collocated, position-based integrators use an equivalent form of the Newton-Euler's equation known as position-based dynamics (PBD) \cite{egt.20171034,10.1007/978-3-030-44051-0_39}, where the main difference is that the discretization is performed in the Euclidean space. Position-based integrators are stable under large timestep sizes but they do not preserve symmetry.

\TE{Contact Handling:} Sequential contact models \cite{guendelman2003nonconvex,mirtich1996impulse} have a significantly limited stability region due to the stiffness of contact forces. Models allowing simultaneous collisions and contacts have larger stability region, especially using implicit time stepper \cite{stewart2000rigid,1249734}. LCP \cite{stewart2000rigid,anitescu2006optimization} and MDP \cite{drumwright2010modeling,kaufman2008staggered} are the most popular implicit formulations for simultaneous frictional contacts. Solving complementary conditions due to LCP is NP-hard and can sometimes be infeasible \cite{painleve2012lois}. The MDP relaxes the complementary constraints by allowing any contact forces in the frictional cone to be feasible. However, the stability region of time-stepper is still limited by the linearization of dynamic systems using either LCP or MDP. In \cite{1249734,10.1007/978-3-030-44051-0_39}, dynamics with frictionless contacts are reformulated as an optimization and linearization can be avoided, but these results cannot be extended to frictional cases. In \cite{8593817}, the frictional force is modified and then reformulated as an optimization, but the modified variant cannot handle static-sliding frictional mode switches. In \cite{10.1145/3272127.3275011}, a modified BDF2 scheme is proposed to achieve second-order accuracy in time-integration under frictional contacts, but linearization is still needed. Unlike these methods, we analyze the feasibility of contact handling without linearization for both normal and frictional forces.

\TE{Generalized Contact Models:} Although the Coulomb frictional model is sufficient for most scenarios involving only rigid objects, other contact models are needed for several reasons. To model the unknown continuous force distribution between a planar object and a flat ground, a general convex wrench space is learned from real-world data in \cite{Zhou2018}. In other works \cite{8794337,maladen2011biophysically,6943251,hu2018moving}, articulated robots are walking on or swimming in deformable environments with granular or compliant materials. Hu \cite{hu2018moving} simulated both the granular material and the robot using fine-grained finite element method, which is more than $1000\times$ slower than a standalone articulated body simulation. In \cite{maladen2011biophysically,6943251}, the Coulomb frictional model is replaced with analytic and empirical force models. Although these models are cheap to compute, they cannot capture the static-sliding frictional mode switches. Zhu \cite{8794337} used a similar approach as \cite{Zhou2018} and learned a robot-pose-dependent contact wrench space. Static-sliding frictional mode switches can thereby be modeled using MDP solver with the learned wrench space as constraints. By extending MDP, our NMDP solver can handle any generalized force models \cite{Zhou2018,8794337} in the form of robot-pose-dependent contact wrench spaces.

%% file: problem.tex
\section{\label{sec:problem}Articulated Body Dynamics}
In this section, we briefly review two prior formulations of articulated body dynamics and their corresponding discretization schemes: the recursive Newton-Euler's equation and position-based dynamics. Both schemes can be extended to derive NMDP solvers.

\subsection{Recursive Newton-Euler's Equation}
The continuous Newton-Euler's equation under generalized coordinates takes the following form:
\begin{align}
\label{eq:NE}
0=H(\theta)\ddot{\theta}+C(\theta,\dot\theta)-\sum_{x\in\mathcal{C}}\FPPN{X}{\theta}(x,\theta)^Tf_x-\tau.
\end{align}
Here $H$ is the generalized mass matrix, $\theta$ is the robot's configuration vector, $C(\theta,\dot\theta)$ is the Coriolis and Centrifugal force, $X(x,\theta)$ is the forward kinematic function bringing a point $x$ from the robot's local coordinates to the global coordinates, $\mathcal{C}$ is a set of points in contact with the environment, $f_x$ is the external force on $x$ in world coordinates, and finally $\tau$ is the joint torque. 
\begin{remark}
We assume contacts are realized by external forces $f_x\in\mathbbm{R}^3$. More general contact models such as \cite{8794337} require external wrenches $f_x\in\mathbbm{R}^6$. In this case, we can replace $\FPPN{X}{\theta}(x,\theta)$ with the Jacobian matrix in $\mathbbm{R}^{6\times|\theta|}$ and all the following analysis applies.
\end{remark}

To discretize a dynamic system, the linear multistep method uses finite difference approximations for all variables. We illustrate this method with first-order finite difference schemes and higher-order schemes can be applied following a similar reasoning. We introduce two variables $\theta_{-}$ and $\theta_{--}$. We assume that $\theta_{-}$ is the robot configuration at current time instance, $\theta_{--}$ is the robot configuration $\Delta t$ seconds before ($\Delta t$ is the timestep size), and $\theta$ is the to-be-predicted robot configuration after $\alpha\Delta t$ seconds. Here $\alpha\in(0,1]$ is an additional parameter for timestep size control and we use subscripts to denote functions that are dependent on $\alpha$. Since NMDP solver requires sufficiently small timestep sizes to converge, we use $\alpha$ to ensure this condition holds. Under these definitions, we can approximate:
\begin{align}
\label{eq:FD}
\dot{\theta}\triangleq\frac{\theta-\theta_{-}}{\alpha\Delta t}\quad
\dot{\theta}_{-}\triangleq\frac{\theta_{-}-\theta_{--}}{\Delta t}\quad
\ddot{\theta}\triangleq\frac{\dot{\theta}-\dot{\theta}_{-}}{\Delta t}.
\end{align}
Plugging these approximations into the Newton-Euler's equation, the forward-Euler integrator takes the following form:
\begin{align}
\label{eq:LNE}
0=H(\theta_{-})\ddot{\theta}-C(\theta_{-},\dot{\theta}_{-})-\sum_{x\in\mathcal{C}}\FPPN{X}{\theta}(x,\theta_{-})^Tf_x-\tau,
\end{align}
which is a linearized dynamic system in $\theta,f_x,\tau$. Instead, the backward-Euler integrator evaluates $H,C$ at time-level $\alpha\Delta t$ instead of current time instance, resulting in a nonlinear system of \eqref{eq:NE} and \eqref{eq:FD}. This system is not guaranteed to have a solution, unless a small enough timestep size is used.

\subsection{Position-Based Dynamics}
PBD reformulates the governing equation-of-motion as:
\begin{align}
\label{eq:PBD}
0=\FPPN{E_\alpha}{\theta}(\theta,f_x),
\end{align}
where we define:
\begin{small}
\begin{equation*}
\begin{aligned}
&E_\alpha(\theta,f)\triangleq I_\alpha(\theta)-\sum_{x\in\mathcal{C}}X(x,\theta)^Tf_x-\theta^T\tau   \\
&I_\alpha(\theta)\triangleq\int_{x\in\mathcal{R}}
\frac{\rho\|X(x,\theta)-(1+\alpha) X(x,\theta_{-})+\alpha X(x,\theta_{--})\|^2}
{2\alpha\Delta t^2}dx,
\end{aligned}
\end{equation*}
\end{small}%
and the integral in $I_\alpha$ is taken over the entire robot $\mathcal{R}$. If we assume that $\theta$ is a continuous trajectory $\theta(t)$ and $\theta=\theta(t+\alpha\Delta t),\theta_{-}=\theta(t),\theta_{--}=\theta(t-\Delta t)$, it has been shown in \cite{10.1007/978-3-030-44051-0_39} that \eqref{eq:PBD} will converge to \eqref{eq:NE} as $\Delta t\to0$. This integral can be evaluated analytically in a similar way as deriving the generalized mass matrix. By comparison with \eqref{eq:NE}+\eqref{eq:FD}, \eqref{eq:PBD} is always solvable under arbitrarily large timestep sizes because it is integrable. In other words, solving for $\theta$ is equivalent to the following optimization:
\begin{align*}
\argmin{\theta}E_\alpha(\theta,f_x).
\end{align*}
Note that we assume $f_x$ is a constant in our derivation for the integrability of $E_\alpha$ (i.e. PBD dynamics can be written as $0=\FPPN{E_\alpha}{\theta}(\theta,f_x)$ for some $E_\alpha$). More generally, PBD can still take an integrable form when the external forces are conservative. In scenarios with dissipative force models such as Coulomb frictional forces, both integrability and PBD's feasibility guarantee are lost, just like Newton-Euler's equation. In this work, we propose an algorithm that solves the system of nonlinearity equations with dissipative force models with guaranteed solvability.

%% file: method.tex
\section{Nonconvex MDP}
Our main idea is to combine backward time-integration and frictional contact force computation. In an MDP solver, the force at each contact point $f_x$ belongs to a convex feasible space. We assume that the feasible space is a polytope with a set of vertices denoted as $v_x^j$ with $j=1,\cdots,V_x$. Here $V_x$ is the number of vertices used to model the polytope at contact point $x$. We assume that all the vertices $v_x^j$ are assembled into a matrix $v_x=\THREE{v_x^1}{\cdots}{v_x^{V_x}}$ so feasible $f_x$ is:
\begin{align}
\label{eq:FORCE}
f_x\in\left\{v_xw_x\middle|w_x\succeq0,\E{1}^Tw_x\leq1\right\},
\end{align}
where $w_x$ is the weights of convex combination and $\E{1}$ is an all-one vector. 
\begin{remark}
The assumption of feasible contact force being a polytope is essential for our convergence proof. Under this assumption, we will extensively use the property that $f_x$ has a bilinear form of $v_xw_x$, where $w_x$ is bounded and $v_x$ is sufficiently smooth.
\end{remark}

\subsection{NMDP Formulation}
When modeling an inelastic rigid contact, $v_x$ is set to the vertices of the linearized frictional cone if $X(x,\theta)$ is in contact or penetrating the environment, and $v_x$ is set to zero otherwise. However, the switch between the in-contact and off-contact state is non-differentiable which is undesirable in applications such as differential dynamic programming \cite{tassa2012synthesis} and trajectory optimization \cite{mordatch2012discovery}. Therefore, we assume that $v_x$ is a robot-pose-dependent, differentiable function $v_x(\theta)$. This formulation is compatible with the recently proposed learning-based granular wrench space model \cite{8794337} and can potentially generalize to other contact models. To determine the weights $w_x$, MDP solves an optimization that minimizes the kinetic energy at time instance $\alpha\Delta t$. Conventional MDP solver uses the linearized dynamic system \prettyref{eq:LNE} and discretizes $v_x$ at $\theta_{-}$, resulting in a QP problem. Instead, our NMDP scheme uses the backward-Euler integrator \prettyref{eq:NE} and discretizes $v_x$ at $\theta$. As a result, we need to solve the following nonlinear constrained optimization:
\begin{equation}
\begin{aligned}
\label{eq:NMDP_NE}
\argmin{\theta,w}&K(\theta)\quad\ST0=G_\alpha(\theta,w)  \\
&K(\theta)\triangleq\frac{1}{2}\dot{\theta}^TH(\theta)\dot{\theta}   \\
&G_\alpha(\theta,w)\triangleq H(\theta)\ddot{\theta}+C(\theta,\dot\theta)-  \\
&\sum_{x\in\mathcal{C}}\FPPN{X}{\theta}(x,\theta)^Tv_x(\theta)w_x-\tau,
\end{aligned}
\end{equation}
where we assume the use of recursive Newton-Euler's equation and $w$ is a concatenation of all $w_x$. In the rest of the paper, we propose two algorithms to solve \eqref{eq:NMDP_NE} and analyze their convergence. 

\subsection{NMDP Solver}
Since \eqref{eq:NMDP_NE} is a general nonlinear constrained optimization, it can be solved using general-purpose optimizers such as the interior point method \cite{mehrotra1992implementation}. However, these methods are not guaranteed to converge to a first-order stationary point due to infeasibility. Instead, we consider two variants of the projected gradient method (PGM), which we prove to converge to a first order stationary point. PGM starts from a feasible initial guess and updates a search direction of $w$ by solving:
\begin{align}
\label{eq:QP}
&\argmin{\Delta\theta,\Delta w}K(\theta+\Delta\theta)\quad
\ST G_\alpha(\theta+\Delta\theta,w+\Delta w)=0 \\
&K(\theta+\Delta\theta)\triangleq K(\theta)+
\FPPN{K}{\theta}^T\Delta\theta+
\frac{1}{2}\Delta\theta\FPPTN{K}{\theta}\Delta\theta\nonumber  \\
&G_\alpha(\theta+\Delta\theta,w+\Delta w)\triangleq
\FPPN{G_\alpha}{\theta}\Delta\theta+
\FPPN{G}{w}\Delta w.\nonumber
\end{align}
If $\FPPN{G_\alpha}{\theta}$ is non-singular, then \eqref{eq:QP} is equivalent to the following QP:
\begin{align}
\label{eq:NMDPL}
\argmin{\Delta w}&
-\FPPN{K}{\theta}^T
\FPPN{G_\alpha}{\theta}^{-1}\FPPN{G}{w}\Delta w+\|w\|^2/\gamma   \\
&\frac{1}{2}\Delta w^T\FPPN{G}{w}^T\FPPN{G_\alpha}{\theta}^{-T}
\FPPTN{K}{\theta}\FPPN{G_\alpha}{\theta}^{-1}\FPPN{G}{w}\Delta w\nonumber   \\
\ST& (w+\Delta w)\succeq0,\mathbbm{1}^T(w+\Delta w)\leq\E{1},\nonumber
\end{align}
where we use $\gamma$ to facilitate line search. The matrix $\mathbbm{1}$ is a concatenation of constraints that $w_x$ sums to less than one on each contact point $x$. After solving for a new $w\gets w+\Delta w$, we update $\theta$ by projecting it to the $G_\alpha(\theta,w)=0$ manifold using the following recursion:
\begin{align}
\label{eq:PROJ}
\theta\gets\theta-\FPPN{G_\alpha}{\theta}^{-1}G_\alpha(\theta,w).
\end{align}
Note that we have only used the first-order derivatives of $G_\alpha$ in \eqref{eq:NMDPL} so the PGM has linear convergence speed at best. Our second version of PGM differs in that we ignore all gradients of the function $v_x$, i.e. zeroth-order update for $v_x$. This requirement is inspired by the recent work \cite{8794337} where the contact wrench space is learned from real-world data. In this case, computing derivatives of $v_x$ involves costly back-propagation through a learning model, e.g. neural networks, sublevel sets of high-order polynomials \cite{Zhou2018}, or radial basis functions \cite{8794337}. Mathematically, the derivatives of $v_x$ only occurs in $\FPPN{G_\alpha}{\theta}$ and we denote its zeroth-order, inexact variant as:
\begin{align*}
\FPPN{\bar{G}_\alpha}{\theta}(\theta,w)\triangleq \FPPN{G_\alpha}{\theta}(\theta,w)+
\sum_{x\in\mathcal{C}}\FPPN{X}{\theta}(x,\theta)^T\FPPN{v_x}{\theta}(\theta)w_x.
\end{align*}
Using $\FPPN{\bar{G}_\alpha}{\theta}$, we derive the following, inexact counterpart of QP (\prettyref{eq:NMDPL}):
\begin{align}
\label{eq:ZONMDPL}
\argmin{\Delta w}&
-\FPPN{K}{\theta}^T
\FPPN{\bar{G}_\alpha}{\theta}^{-1}\FPPN{G}{w}\Delta w+\|w\|^2/\gamma   \\
&\frac{1}{2}\Delta w^T\FPPN{G}{w}^T\FPPN{\bar{G}_\alpha}{\theta}^{-T}
\FPPTN{K}{\theta}\FPPN{\bar{G}_\alpha}{\theta}^{-1}\FPPN{\bar{G}_\alpha}{w}\Delta w\nonumber   \\
\ST& (w+\Delta w)\succeq0,\mathbbm{1}^T(w+\Delta w)\leq\E{1},\nonumber
\end{align}
and the following, inexact counterpart of manifold projection (\prettyref{eq:PROJ}):
\begin{align}
\label{eq:ZOPROJ}
\theta\gets\theta-\FPPN{\bar{G}_\alpha}{\theta}^{-1}G_\alpha(\theta,w).
\end{align}
The pipeline of both first- and zeroth-order PGM is outlined in \prettyref{alg:PGM}.
\vspace{-5px}
\begin{algorithm}[ht]
\caption{\label{alg:PGM} (First- / Zeroth-) Order PGM($\alpha,\Delta t,\theta_{-},\theta_{--}$)}
\begin{small}
\begin{algorithmic}[1]
\State $w^0\gets0,\theta^0\gets\theta_{-},\gamma^0\gets1,\eta>1$
\While{$\|G_\alpha(\theta^0,w^0)\|\neq0$}
\State Compute \eqref{eq:PROJ} or \eqref{eq:ZOPROJ} ($\theta=\theta^0,w=w^0$)
\EndWhile
\For{$k=1,\cdots$}
\State Solve \eqref{eq:NMDPL} or \eqref{eq:ZONMDPL} ($\theta=\theta^{k-1},w=w^{k-1}$) for $w^k$
\State $\theta^k\gets\theta^{k-1}$
\While{$\|G_\alpha(\theta^k,w^k)\|\neq0$}
\State Compute \eqref{eq:PROJ} or \eqref{eq:ZOPROJ} ($\theta=\theta^k,w=w^k$)
\EndWhile
\If{$K(\theta^k)>K(\theta^{k-1})$}
\State $\gamma\gets\eta\gamma,\theta^k\gets\theta^{k-1},w^k\gets w^{k-1}$
\Else
\State $\gamma\gets\gamma/\eta$
\If{$\|\theta^k-\theta^{k-1}\|_\infty<\epsilon$}
\State Return $\theta^k,w^k$
\EndIf
\EndIf
\EndFor
\end{algorithmic}
\end{small}
\end{algorithm}
\vspace{-10px}

%% file: analysis.tex
\section{Convergence Analysis}
\begin{figure*}[th]
\centering
\vspace{-10px}
\scalebox{.5}{
\begin{tabular}{cc}
\includegraphics[width=0.99\textwidth]{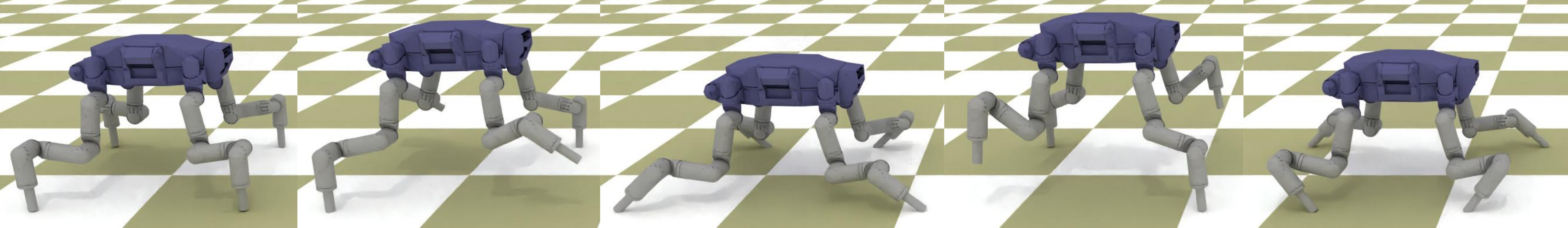}&
\includegraphics[width=0.99\textwidth]{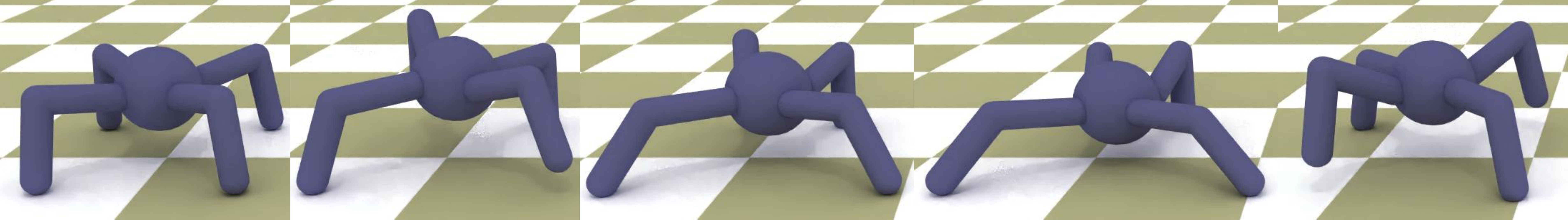}\\
\includegraphics[width=0.99\textwidth]{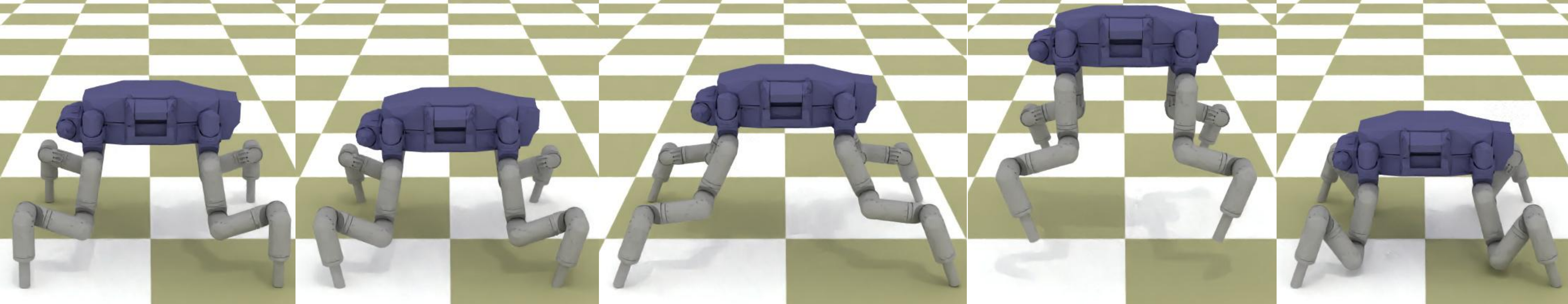}&
\includegraphics[width=0.99\textwidth]{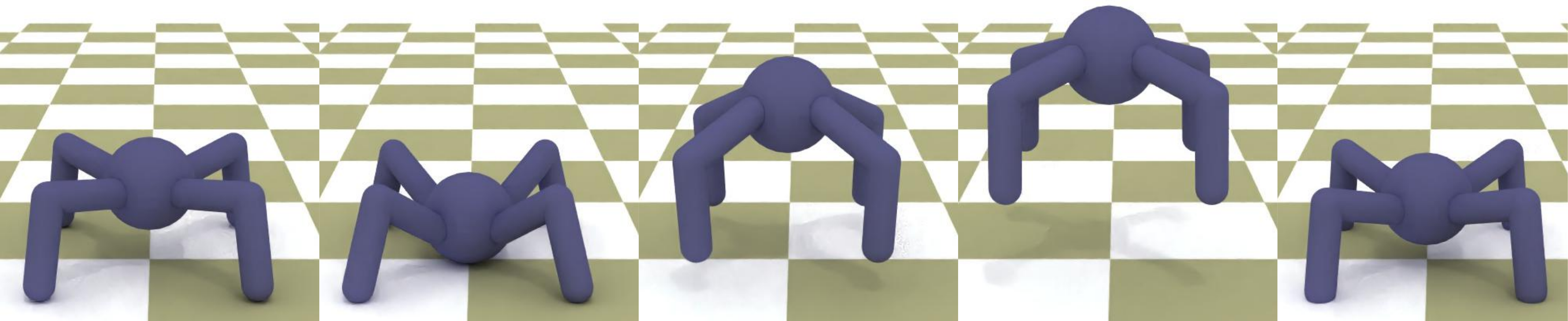}\\
\end{tabular}}
\caption{\label{fig:results} \small{The robots' walking (top) and jumping (bottom) trajectories tracked using the stable PD controller and simulated using our NMDP solver, where we use $\Delta t=0.05$s.}}
\vspace{-15px}
\end{figure*}
We analyze the convergence of \prettyref{alg:PGM} in both first- and zeroth-order cases. PGM cannot proceed if $\FPPN{G_\alpha}{\theta}$ is rank-deficient and does not have an inverse. In addition, the manifold projection substeps in PGM can diverge without using line-search strategies. Finally, the outer-loop of the zeroth-order PGM can fail to converge by using an inexact gradient. We take the following three assumptions to show that first-order PGM and zeroth-order manifold projection are well-defined and convergent:
\begin{assume}
\label{ass:XSmooth}
$X\in\mathcal{C}^\infty$.
\end{assume}
\begin{assume}
\label{ass:SigmaMin}
$\sigma_{min}\left[\int_{x\in\mathcal{R}}\frac{\rho}{\Delta t^2}
\FPPN{X}{\theta}^T\FPPN{X}{\theta}dx\right](\theta_{-})\geq\sigma_X>0$.
\end{assume}
\begin{assume}
\label{ass:VSmooth}
$\FPPTTT{v_x}{\theta}\in\mathcal{C}^0$.
\end{assume}
To show that zeroth-order PGM is also convergent, we need an additional assumption:
\begin{assume}
\label{ass:GwFullrank}
$\FPPN{G}{w}$ has full row rank.
\end{assume}
\begin{remark}
\prettyref{ass:XSmooth} and \prettyref{ass:SigmaMin} can be satisfied by choosing appropriate parameterizations of robot joints. When all the hinge joints are parameterized using Euler angles, then \prettyref{ass:XSmooth} is satsified. \prettyref{ass:SigmaMin} requires that the kinetic energy $K(\theta)$ is strictly convex at $\theta_{-}$. In other words, for any infinitesimal perturbation $\delta\theta$, there must be some points $x\in\mathcal{R}$ undergoing infinitesimal movements $\delta x$ with non-vanishing $\lim_{\delta\theta\to0}\delta x/\delta\theta$. This assumption can only be violated when the robot is suffering from Gimbal lock of Euler angles, which can be easily resolved by moving the singular point away from the current configuration.
\end{remark}
\begin{remark}
\prettyref{ass:VSmooth} indicates that NMDP can only handle smooth contact force models. Stiff and penetration-free contacts between two rigid objects cannot be handled by our method. However, smooth force models are essential for gradient-based motion planning and control \cite{tassa2012synthesis,mordatch2012discovery,10.1145/3309486.3340246}. In addition, stiff contacts can be approximated by smooth contacts. For example, a linearized frictional cone is a polytope with vertices being $v_x^i=n+t^i\mu$, where $n$ is the contact normal and $t^i$ is a direction on the tangential plane. We modify this definition to satisfy \prettyref{ass:VSmooth} by setting $v_x^i=d(X(x,\theta))^3(n+t^i\mu)$, where $d(X(x,\theta))$ is the penetration depth of the contact point $x$. 
This force model can be made arbitrarily stiff by scaling $v_x^i$ with a big constant.
\end{remark}
\begin{remark}
\prettyref{ass:GwFullrank} is a major limitation of the zeroth-order PGM. Note that $v_x=0$ and thus $\FPPN{G}{w}$ has a zero rank when a contact point $x$ does not penetrate the environment. However, once the penetration depth becomes non-zero, then $v_x\neq0$ and $\FPPN{G}{w}$ can have a full row rank. \prettyref{ass:GwFullrank} disallows such jumps in the rank of $\FPPN{G}{w}$. A simple workaround is to slightly modify the contact model by allowing small, non-zero contact forces even when robot is not in contact with the environment. One method to satisfy \prettyref{ass:GwFullrank} is to set $v_x^i=(d(X(x,\theta))^3+\zeta)(n+t^i\mu)$, where $\zeta$ is a small positive constant. 
In our implementation, we ignore \prettyref{ass:GwFullrank} and have never observed divergence behavior due to the violation of this assumption.
\end{remark}
Under the above assumptions, our main results are:
\begin{theorem}[First-Order PGM Convergence]
\label{Thm:PGM}
Assuming \prettyref{ass:XSmooth}, \prettyref{ass:SigmaMin}, \prettyref{ass:VSmooth}, there exists $\alpha_4>0$, such that for all $\alpha\leq\alpha_4$, the first order \prettyref{alg:PGM} will generate a monotonically decreasing sequence of $\{K(\theta^k)\}$ where each $\theta^k$ satisfies $G(\theta^k,w^k)=0$.
\end{theorem}
\begin{theorem}[Zeroth-Order PGM Convergence]
\label{Thm:ZOPGM}
Assuming \prettyref{ass:XSmooth}, \prettyref{ass:SigmaMin}, \prettyref{ass:VSmooth}, \prettyref{ass:GwFullrank}, there exists $\alpha_5>0$, such that for all $\alpha\leq\alpha_5$, the zeroth-order \prettyref{alg:PGM} will generate a monotonically decreasing sequence of $\{K(\theta^k)\}$ where each $\theta^k$ satisfies $G(\theta^k,w^k)=0$.
\end{theorem}
The proofs of these results are deferred to \prettyref{sec:appendix}. Both results imply that timestep sizes cannot be arbitrarily large, otherwise PGM can fail to converge. In addition, the divergence behavior of PGM can only happen in the manifold projection substep when the norm $\|G_\alpha\|,\|\bar{G}_\alpha\|$ does not decrease after apply \prettyref{eq:PROJ} or \prettyref{eq:ZOPROJ}, which is an indicator of the use of smaller timestep sizes. As a result, we can design a robust articulated body simulator using adaptive timestep control as illustrated in \prettyref{alg:adaptive}. 

\prettyref{alg:adaptive} starts by time-integrating using $\alpha=1$. If PGM diverges, we cut the timestep size by half, i.e. setting $\alpha=0.5$ and recurse. Note that $\text{simulate}(\alpha,\Delta t,\theta_-,\theta_{--})$ implies 1) the last timestep size is $\Delta t$ and 2) the desired next time instance is $\alpha\Delta t$ ahead. If PGM diverges, we slice timestep size by half and call $\text{simulate}(\alpha/2,\Delta t/2,\theta_-,\theta_{--})$ for the first half. However, further subdivision might happen for the first half due to recursion, so we return the last timestep size, say $\Delta t^*$. Next, we time integrate the second half. Given that last timestep size is $\Delta t^*$ and our desired $\alpha^*\Delta t^*=\alpha\Delta t/2$, we must have $\alpha^*=(\alpha\Delta t)/(2\Delta t^*)$.
\vspace{-5px}
\begin{algorithm}[ht]
\caption{\label{alg:adaptive} simulate($\alpha,\Delta t,\theta_{-},\theta_{--}$)}
\begin{small}
\begin{algorithmic}[1]
\State $\theta,w\gets$PGM($\alpha,\Delta t,\theta_{-},\theta_{--}$)
\If{Converged}
\State Return $\alpha\Delta t,\theta,\theta_{-}$
\Else
\State $\Delta t^*,\theta^*,\theta_{-}^*\gets$
simulate($\alpha/2,\Delta t,\theta_{-},\theta_{--}$)
\State $\alpha^*\gets(\alpha\Delta t)/(2\Delta t^*)$
\State $\Delta t^{**},\theta^{**},\theta_{-}^{**}\gets$
simulate($\alpha^*,\Delta t^*,\theta^*,\theta_{-}^*$)
\State Return $\Delta t^{**},\theta^{**},\theta_{-}^{**}$
\EndIf
\end{algorithmic}
\end{small}
\end{algorithm}
\vspace{-10px}

\subsection{NMDP Working with PBD}
Our analysis and formulation assumes the use of Newton-Euler's equation. An equivalent form of NMDP can be formulated for the position-based dynamics via a new definition of $K(\theta)$ and $G_\alpha$ as follows:
\begin{equation}
\begin{aligned}
\label{eq:NMDP_PBD}
\argmin{\theta,w}&K(\theta)\quad\ST0=G_\alpha(\theta,w)    \\
K(\theta)\triangleq&\int_{x\in\mathcal{R}}
\frac{\rho\|X(x,\theta)-X(x,\theta_{-})\|^2}{2\Delta t^2}dx \\
G_\alpha\triangleq&\FPPN{I_\alpha}{\theta}(\theta)-
\sum_{x\in\mathcal{C}}\FPPN{X}{\theta}(x,\theta)^Tv_x(\theta)w_x-\tau,
\end{aligned}
\end{equation}
and all the convergence analysis applies to \prettyref{eq:NMDP_NE} and \prettyref{eq:NMDP_PBD} alike. We refer readers to \ifarxiv \prettyref{sec:appendix} \else \cite{} \fi for all the proof.

%% file: results.tex
\section{\label{sec:results}Evaluations}
We evaluate the performance of NMDP in various scenarios. We implement both versions of NMDP (\prettyref{eq:NMDP_NE} and \prettyref{eq:NMDP_PBD}) using C++ and Eigen \cite{eigenweb}, where the optimizations can be solved using both first- and zeroth-order PGM. All the matrix inversions in manifold projection are solved by a rank-revealing LU factorization. As long as the factorization detects that the matrix is near singular (i.e. \prettyref{ass:SigmaMin} is violated) or the norm $\|G_\alpha\|,\|\bar{G}_\alpha\|$ does not decrease, we restart PGM with smaller timestep sizes. In each outer loop of PGM, a QP is solved and the problem data of these QP are quite similar. We use the parametric QP solver \cite{ferreau2014qpoases} that can make use of these similarities to accelerate computation. Finally, we set $\epsilon=10^{-6},\eta=1.5,\zeta=10^{-3}$.

As illustrated in \prettyref{fig:results}, we conduct experiments on the Robosimian by having different simulators to track a prescribed robot walking or jumping trajectory using the stable PD controller \cite{tan2011stable}. The stable PD controller is consistent with the backward-Euler integrator, which uses $\theta,\dot{\theta}$ instead of $\theta_{-},\dot{\theta}_{-}$ as the target state to be tracked. We compare the performance of the following simulators: 
\begin{itemize}
\item NE-NMDP-PGM/NE-NMDP-ZOPGM: 
\prettyref{eq:NMDP_NE} solved using first-/zeroth-order PGM.
\item PBD-NMDP-PGM/PBD-NMDP-ZOPGM: 
\prettyref{eq:NMDP_PBD} solved using first-/zeroth-order PGM.
\item NE-MDP: linearized Newton-Euler's equation with contact forces solved using MDP.
\end{itemize}

\begin{figure*}[th]
\centering
\newcolumntype{?}{!{\hspace{2px} \vrule width 1pt}}
\setlength{\tabcolsep}{0pt}
\tabulinesep=0pt
\begin{tabu}{ccccc}
\includegraphics[width=0.195\textwidth]{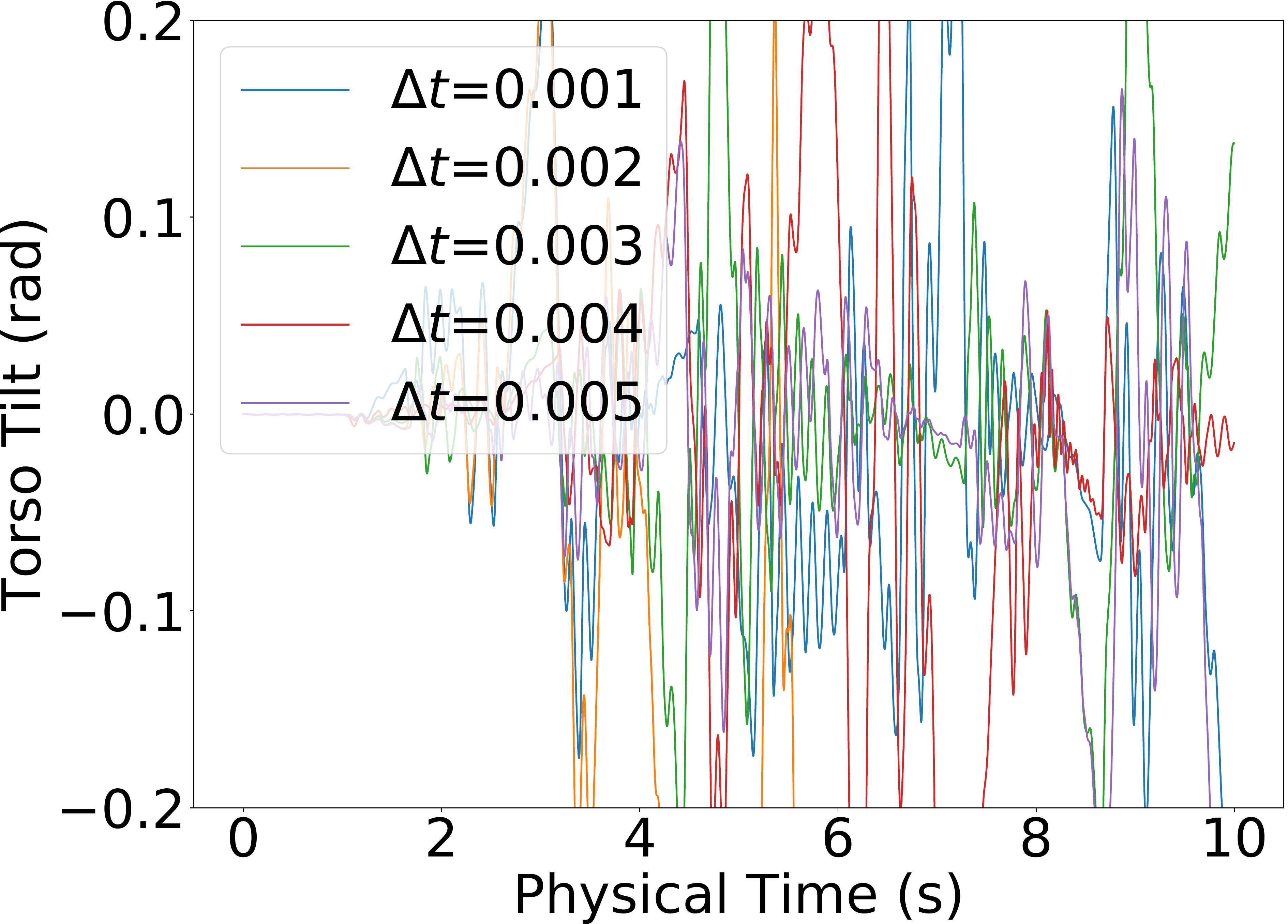}\put(-75,25){(a)}&
\includegraphics[width=0.195\textwidth]{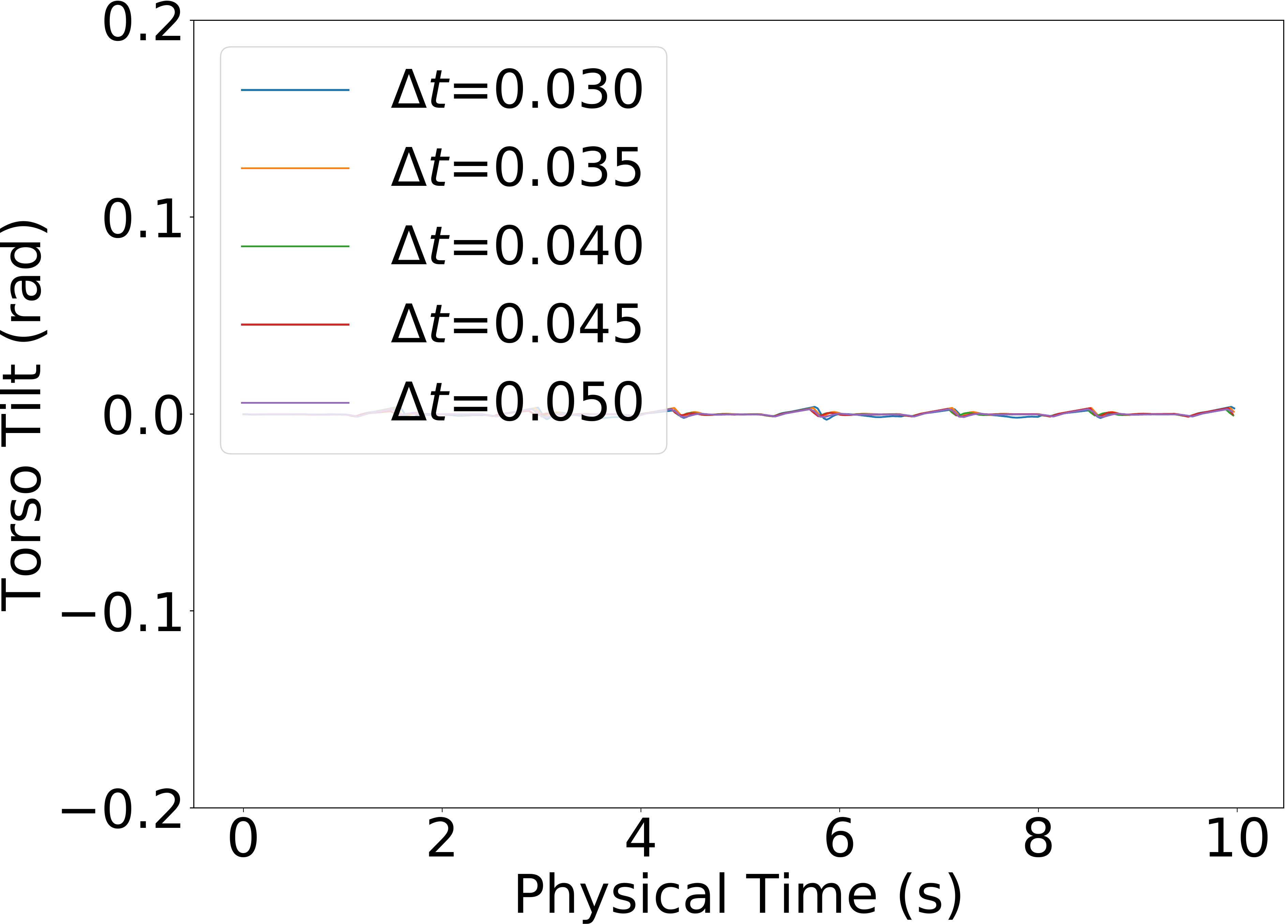}\put(-75,25){(b)}&
\includegraphics[width=0.195\textwidth]{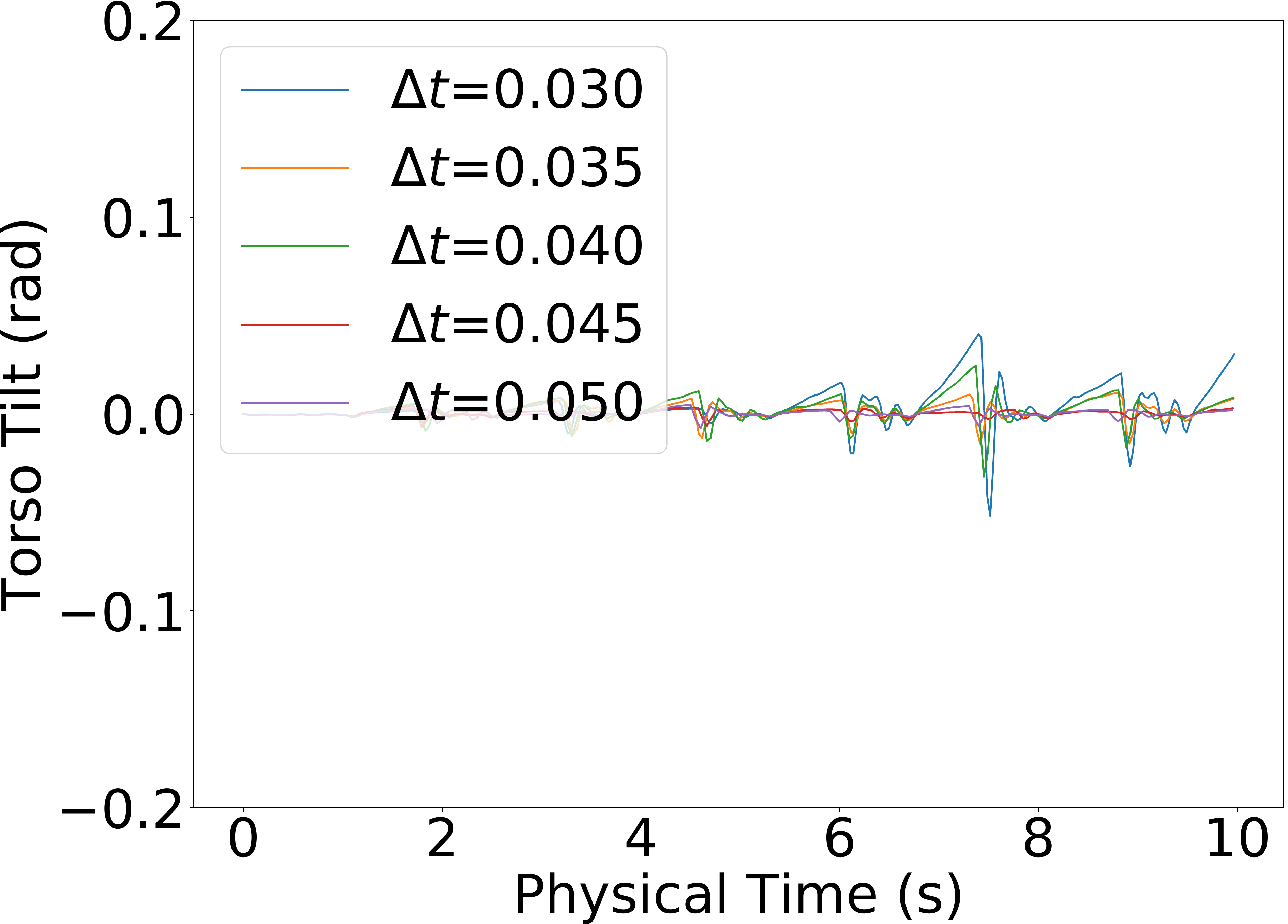}\put(-75,25){(c)}&
\includegraphics[width=0.195\textwidth]{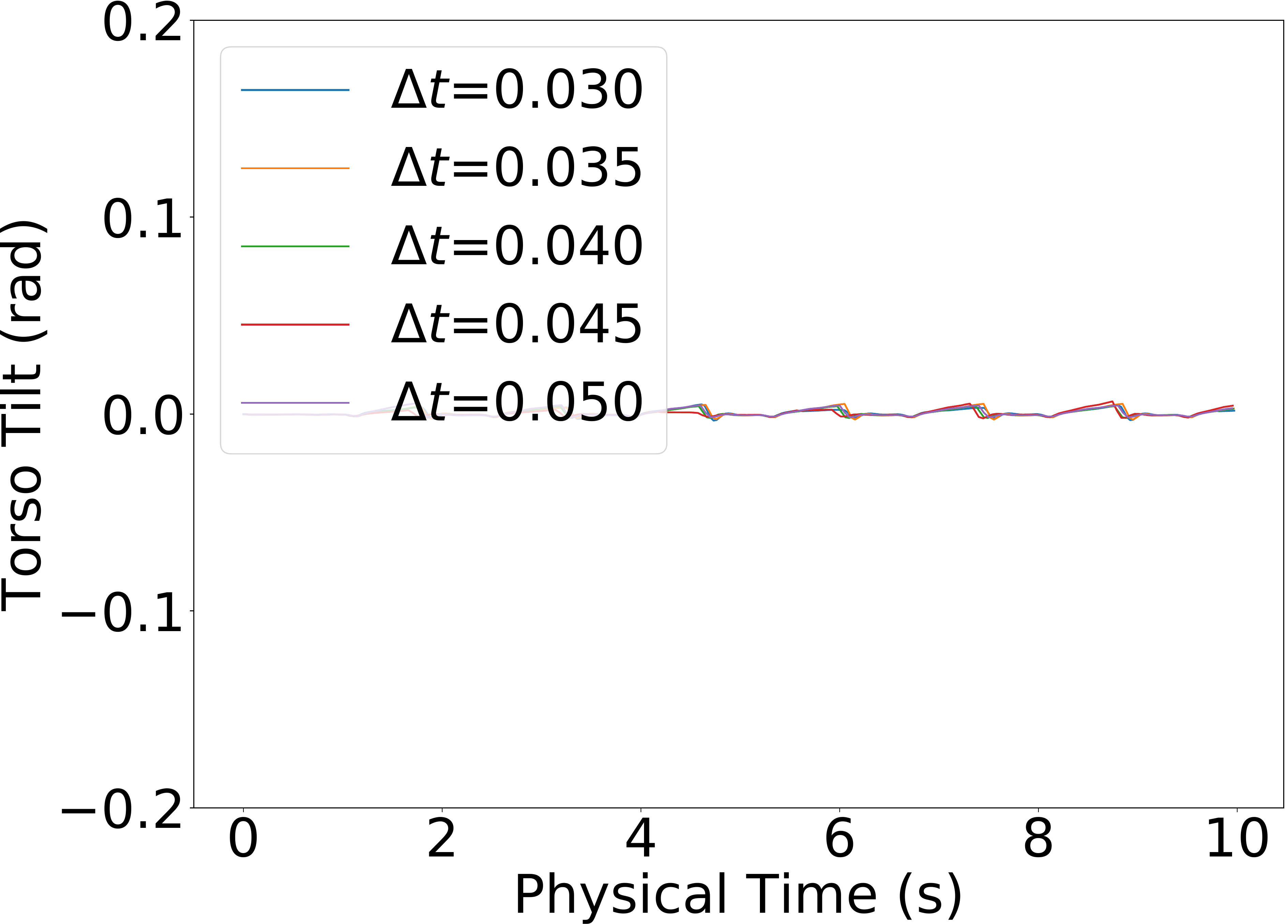}\put(-75,25){(d)}&
\includegraphics[width=0.195\textwidth]{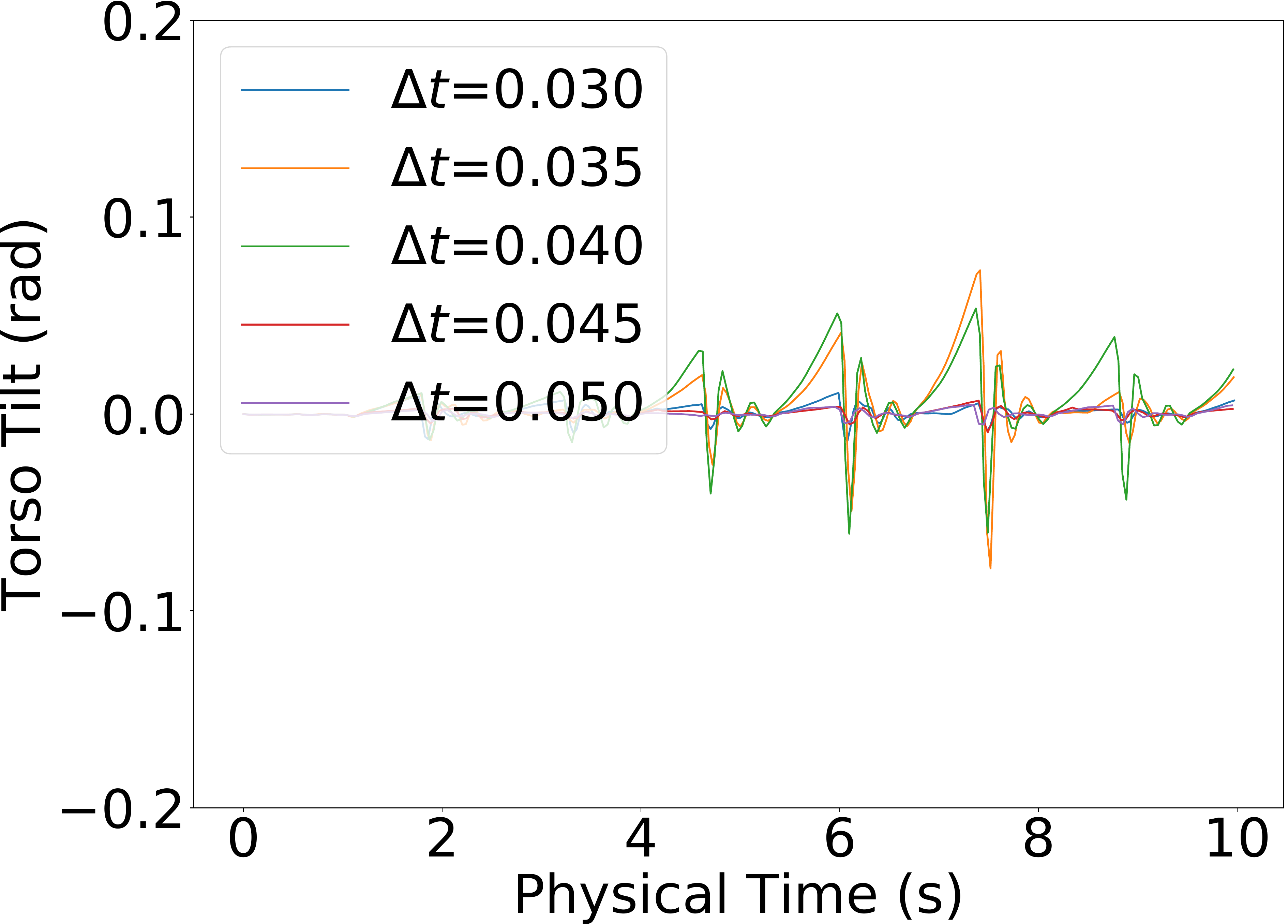}\put(-75,25){(e)}
\end{tabu}
\vspace{-8px}
\caption{\label{fig:jump} \small{The torso tilt of the tracked jumping trajectory. (a): NE-MDP (b): NE-NMDP-PGM (c): NE-NMDP-ZOPGM (d): PBD-NMDP-PGM (e): PBD-NMDP-ZOPGM}}
\vspace{-10px}
\end{figure*}
\begin{figure*}[th]
\centering
\newcolumntype{?}{!{\hspace{2px} \vrule width 1pt}}
\setlength{\tabcolsep}{0pt}
\tabulinesep=0pt
\begin{tabu}{ccccc}
\includegraphics[width=0.195\textwidth]{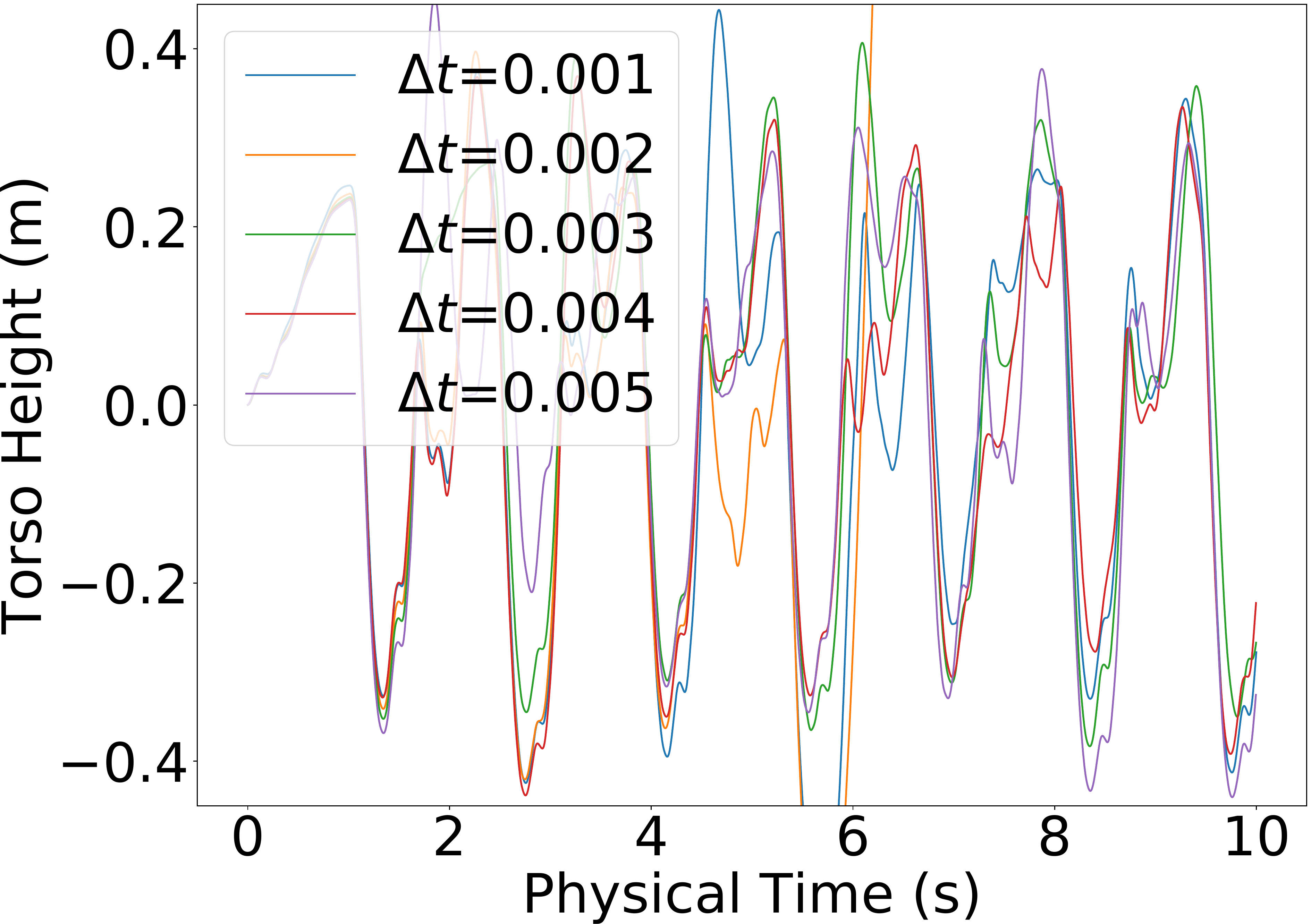}\put(-25,56){(a)}&
\includegraphics[width=0.195\textwidth]{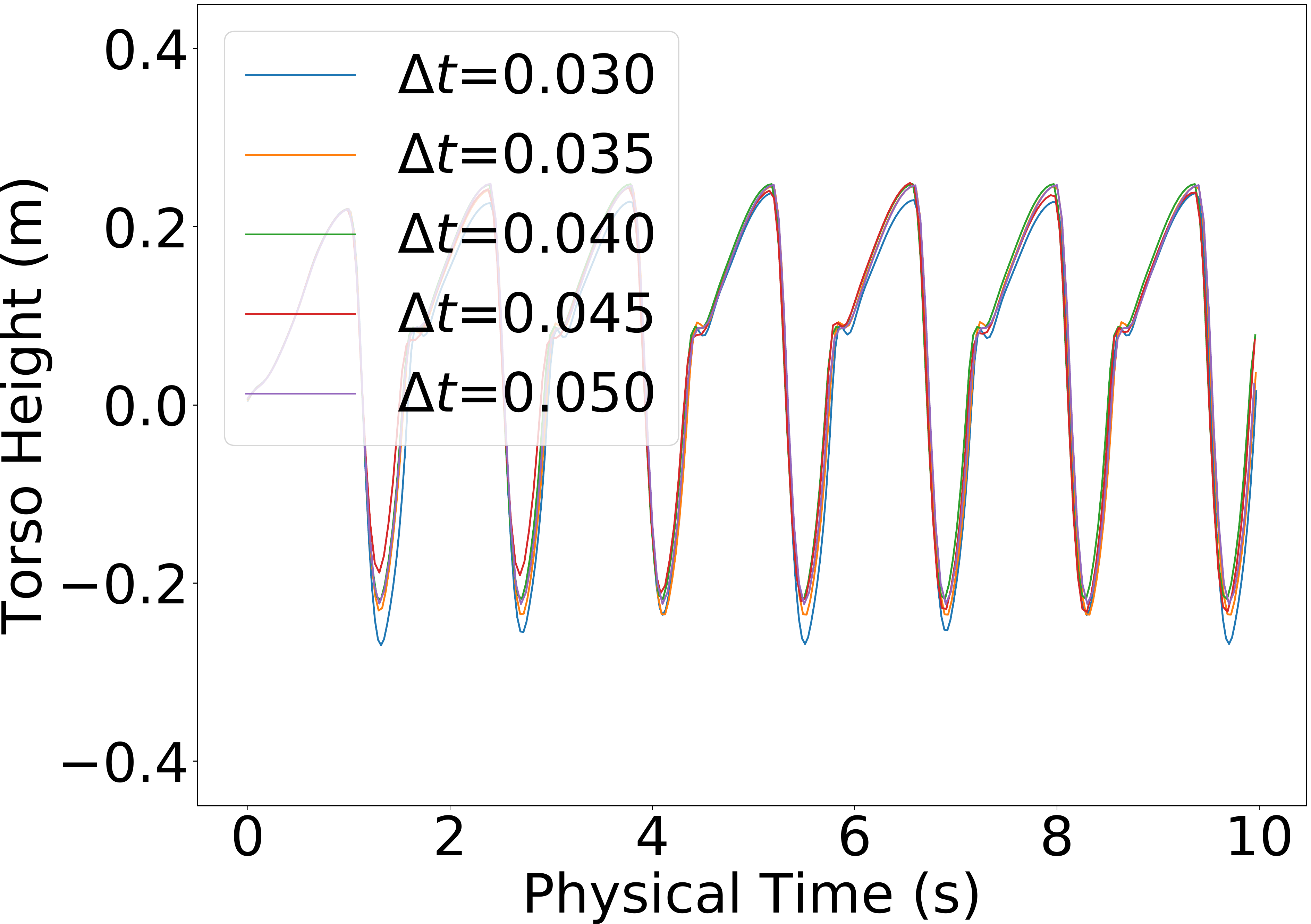}\put(-25,56){(b)}&
\includegraphics[width=0.195\textwidth]{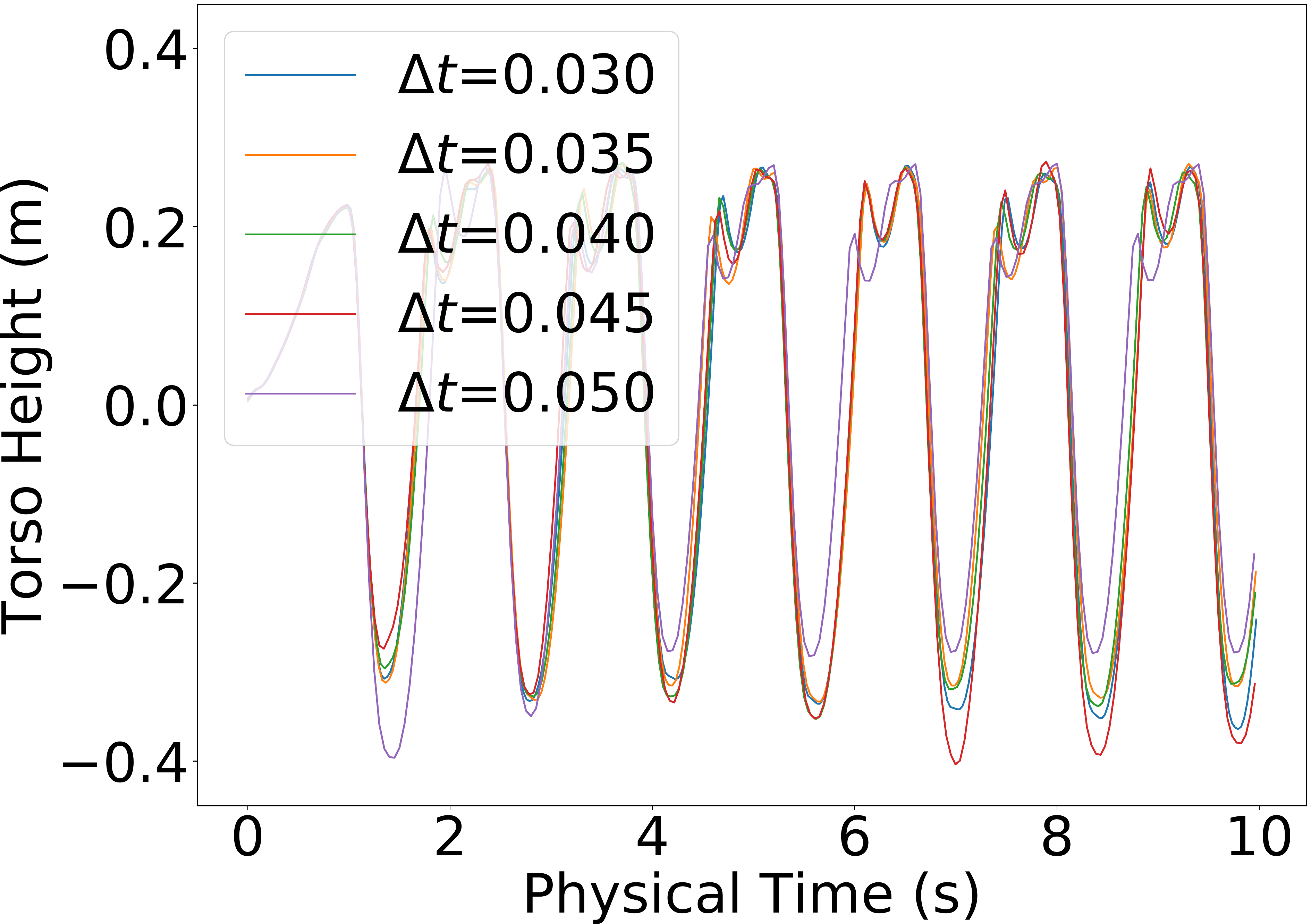}\put(-25,56){(c)}&
\includegraphics[width=0.195\textwidth]{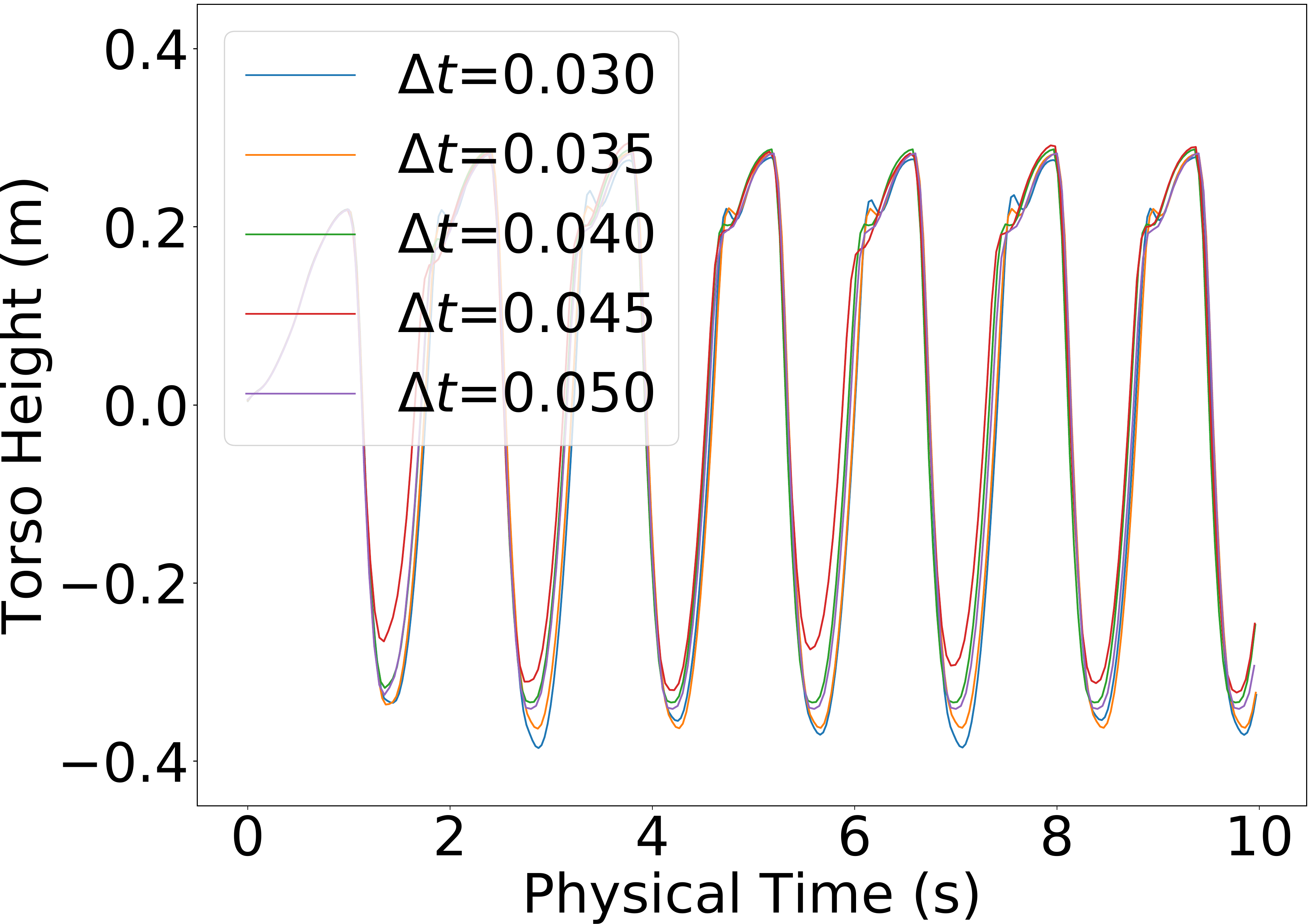}\put(-25,56){(d)}&
\includegraphics[width=0.195\textwidth]{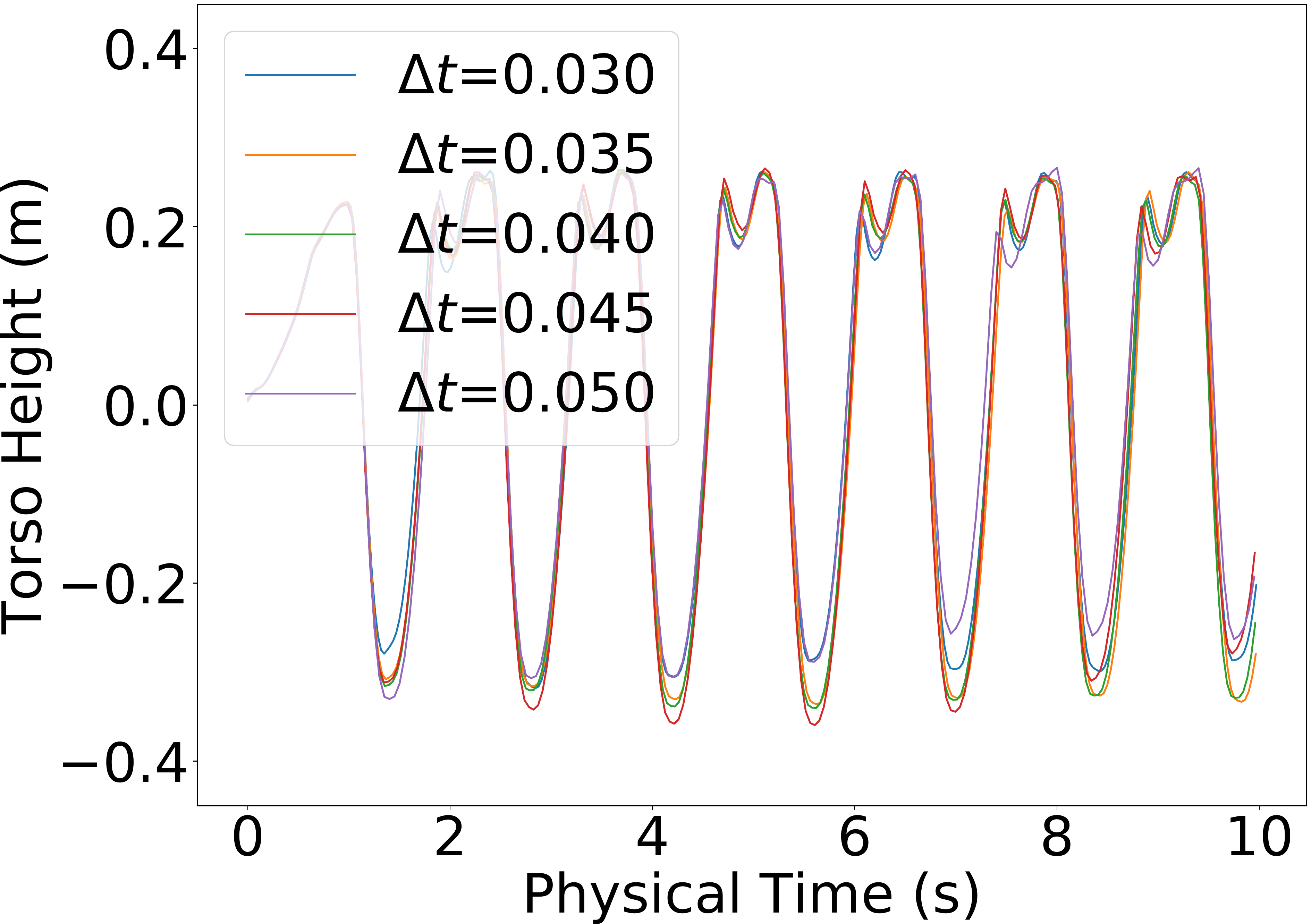}\put(-25,56){(e)}
\end{tabu}
\vspace{-8px}
\caption{\label{fig:height} \small{The torso height of the tracked jumping trajectory. (a): NE-MDP (b): NE-NMDP-PGM (c): NE-NMDP-ZOPGM (d): PBD-NMDP-PGM (e): PBD-NMDP-ZOPGM}}
\vspace{-10px}
\end{figure*}
\begin{figure*}[th]
\centering
\newcolumntype{?}{!{\hspace{2px} \vrule width 1pt}}
\setlength{\tabcolsep}{0pt}
\tabulinesep=0pt
\begin{tabu}{ccccc}
\includegraphics[width=0.195\textwidth]{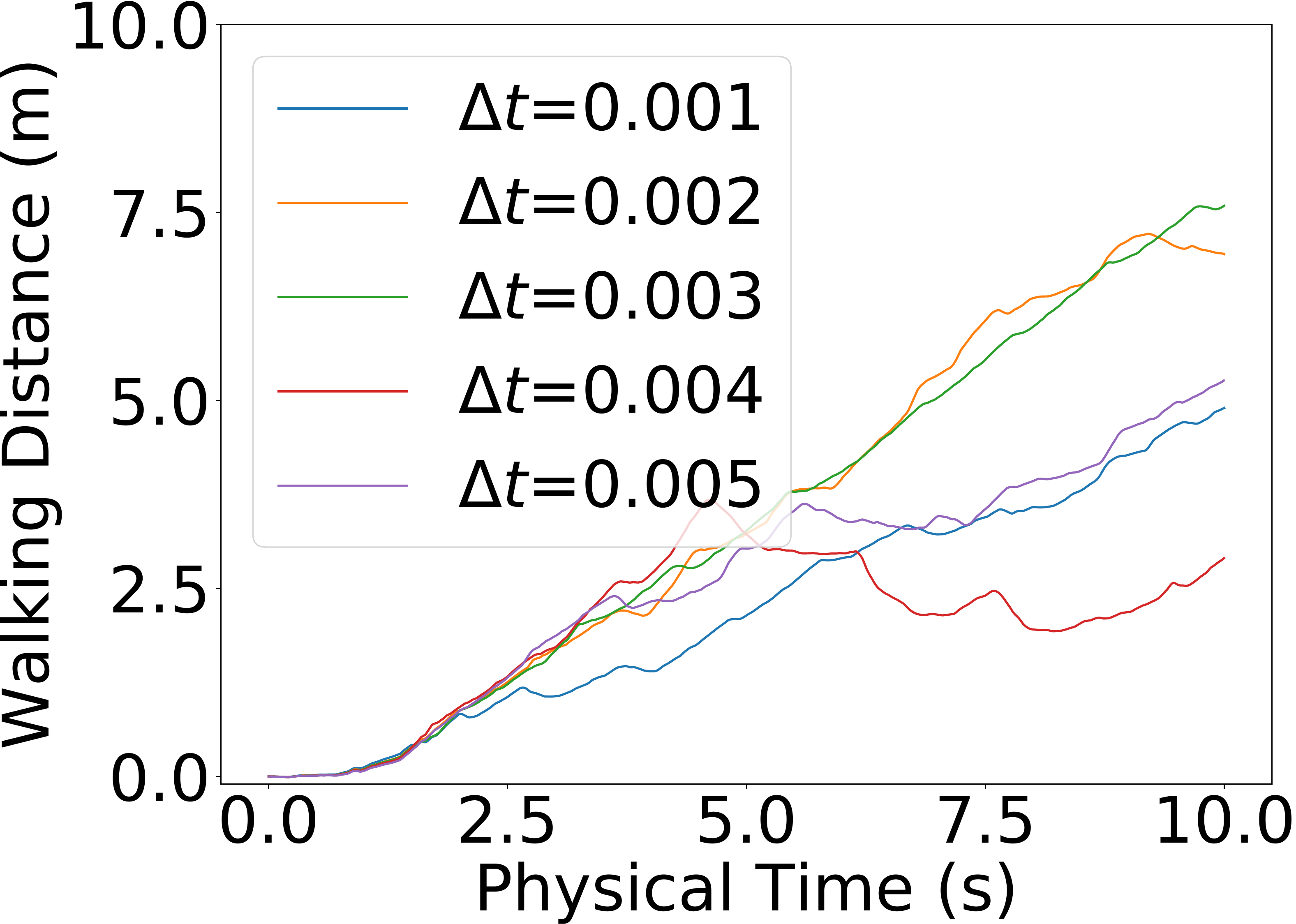}\put(-75,20){(a)}&
\includegraphics[width=0.195\textwidth]{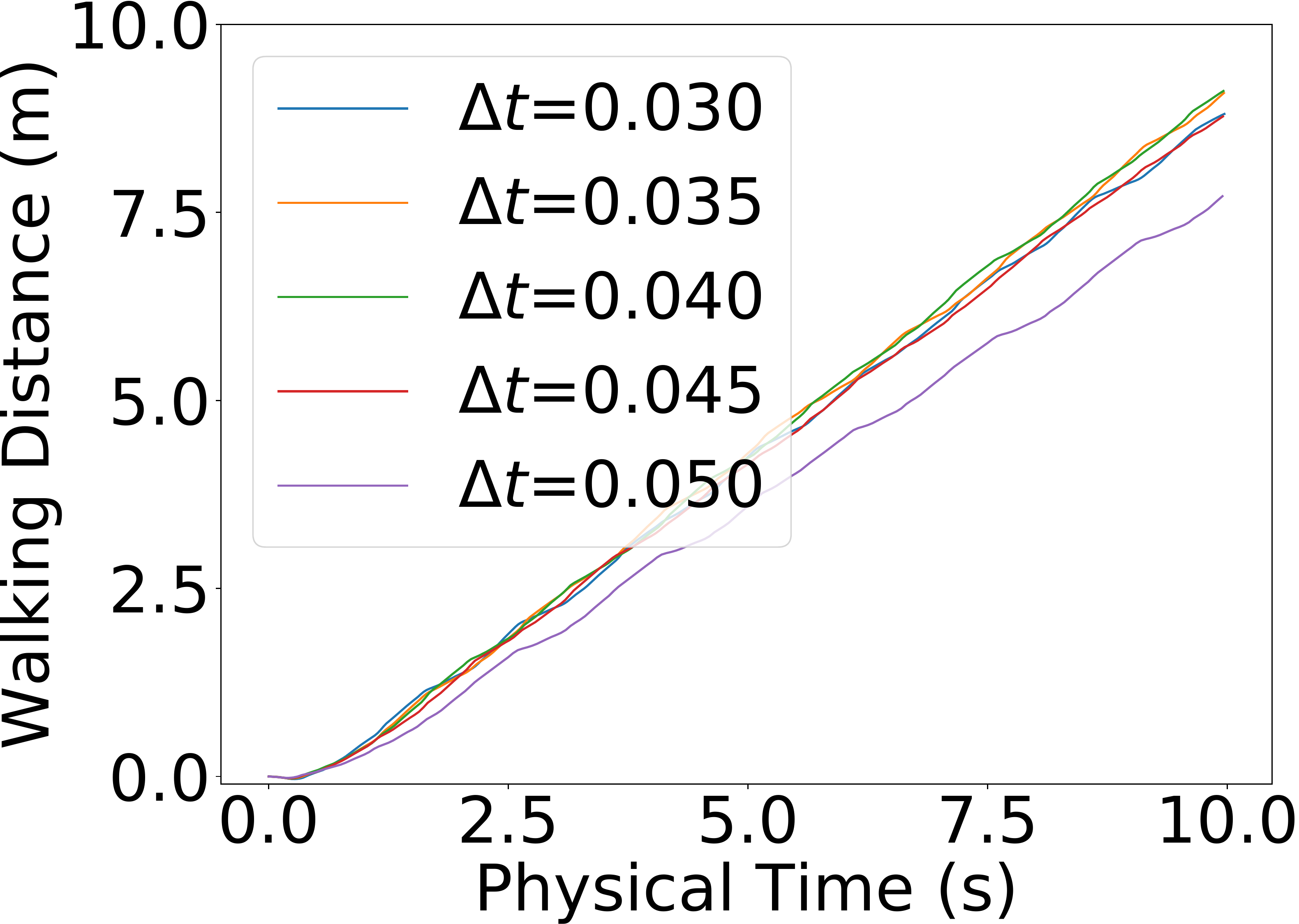}\put(-75,20){(b)}&
\includegraphics[width=0.195\textwidth]{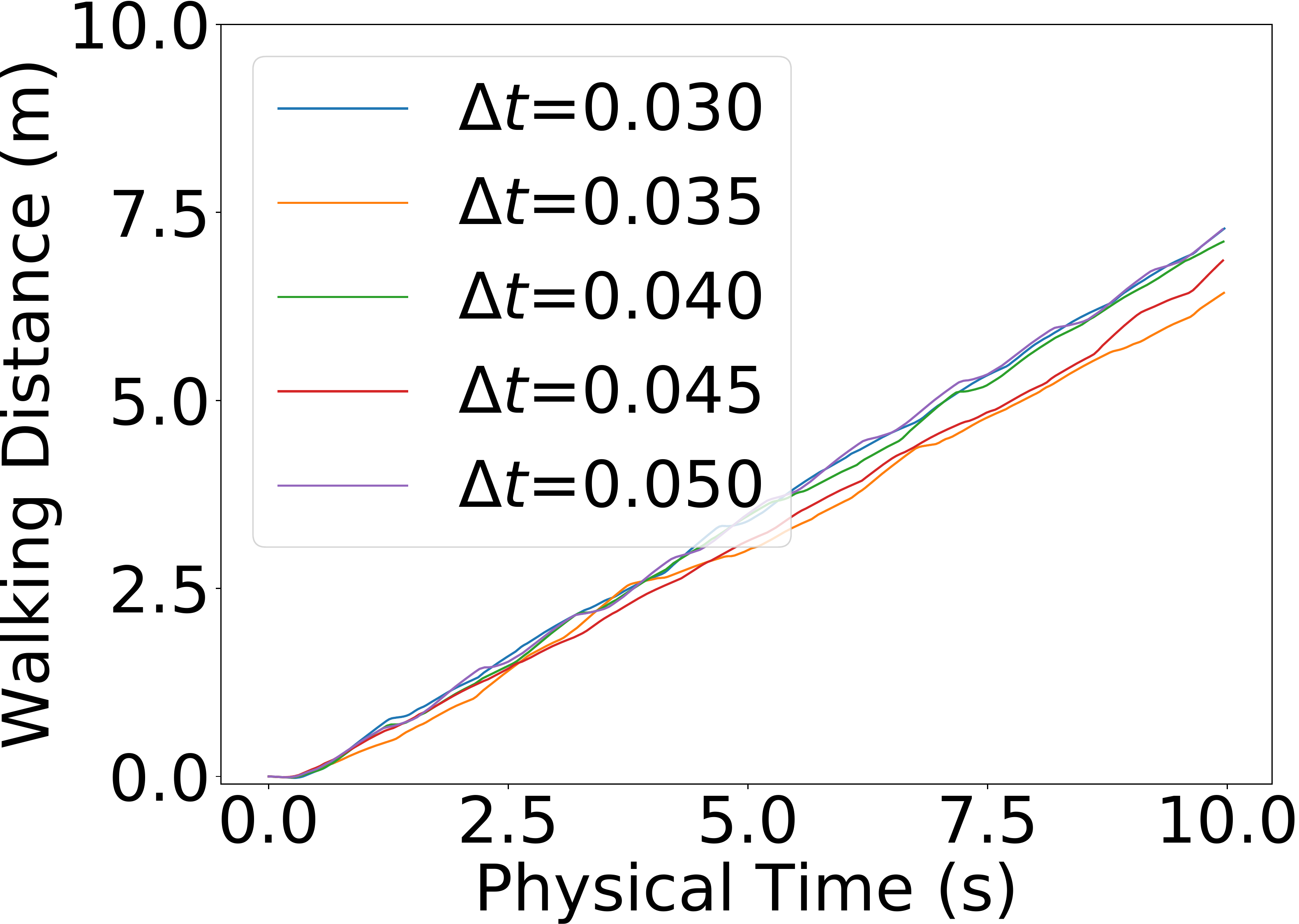}\put(-75,20){(c)}&
\includegraphics[width=0.195\textwidth]{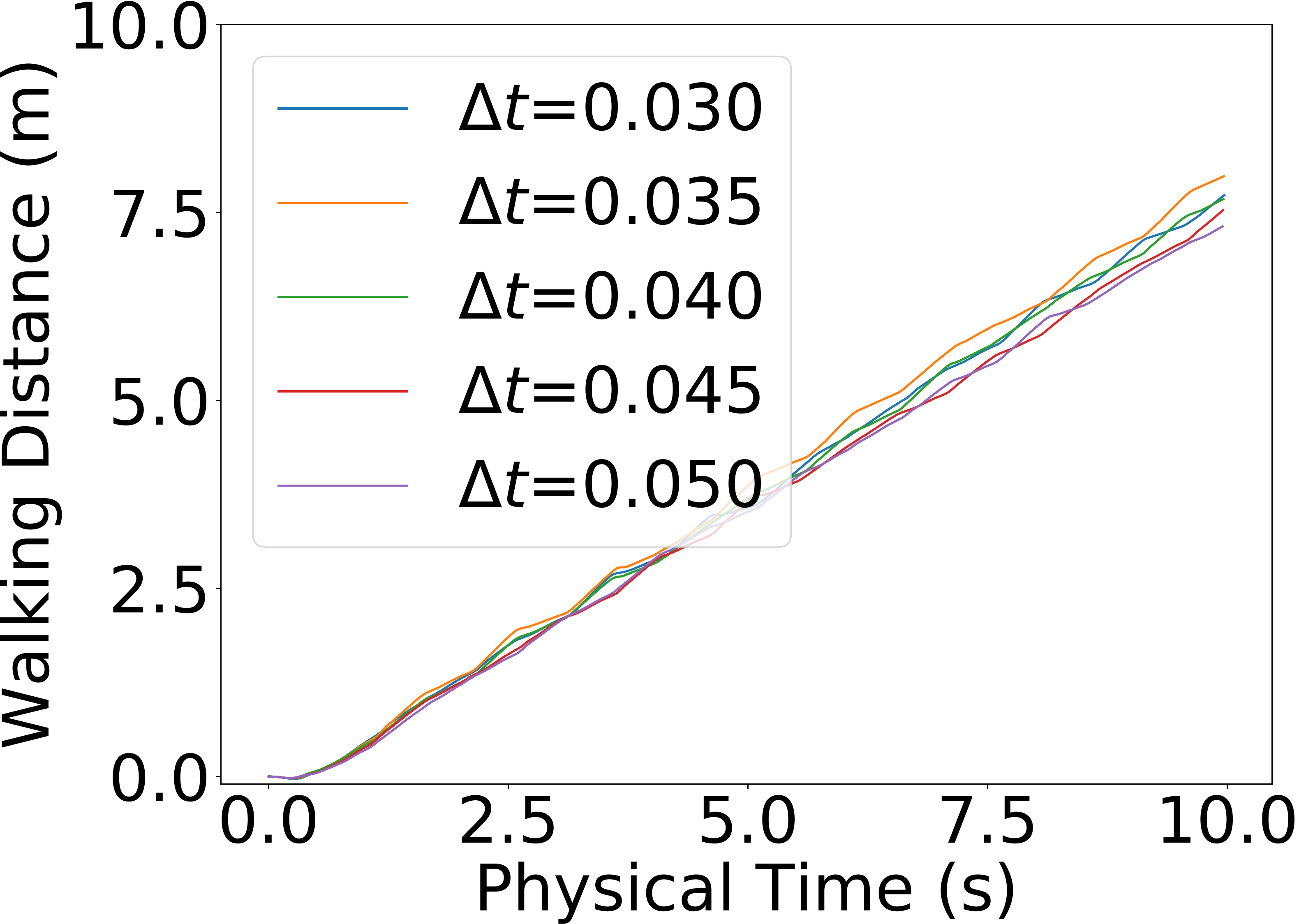}\put(-75,20){(d)}&
\includegraphics[width=0.195\textwidth]{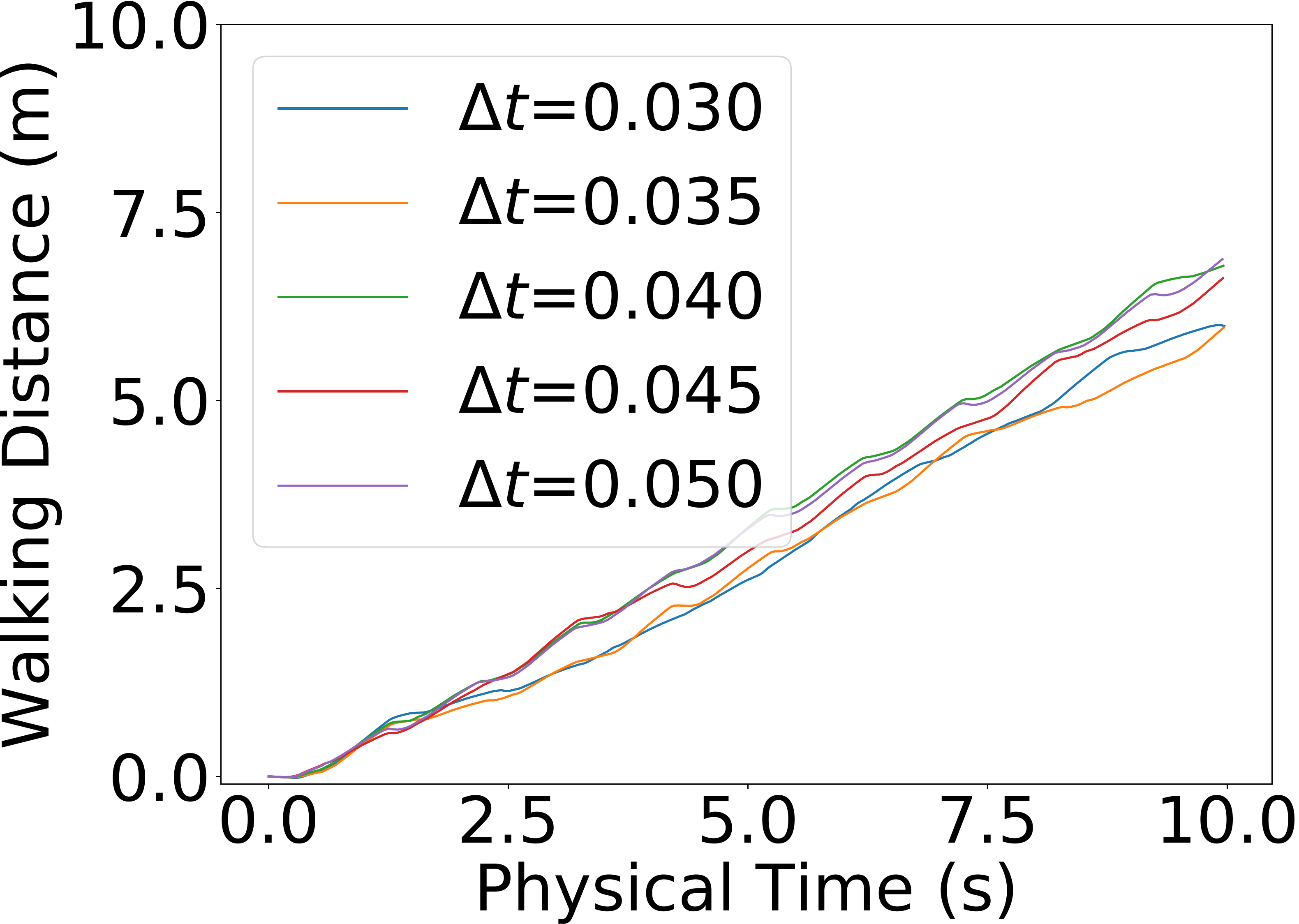}\put(-75,20){(e)}
\end{tabu}
\vspace{-8px}
\caption{\label{fig:distance} \small{The distance of the tracked walking trajectory. (a): NE-MDP (b): NE-NMDP-PGM (c): NE-NMDP-ZOPGM (d): PBD-NMDP-PGM (e): PBD-NMDP-ZOPGM}}
\vspace{-10px}
\end{figure*}
\begin{figure*}[th]
\centering
\newcolumntype{?}{!{\hspace{2px} \vrule width 1pt}}
\setlength{\tabcolsep}{0pt}
\tabulinesep=0pt
\begin{tabu}{ccccc}
\includegraphics[width=0.2\textwidth]{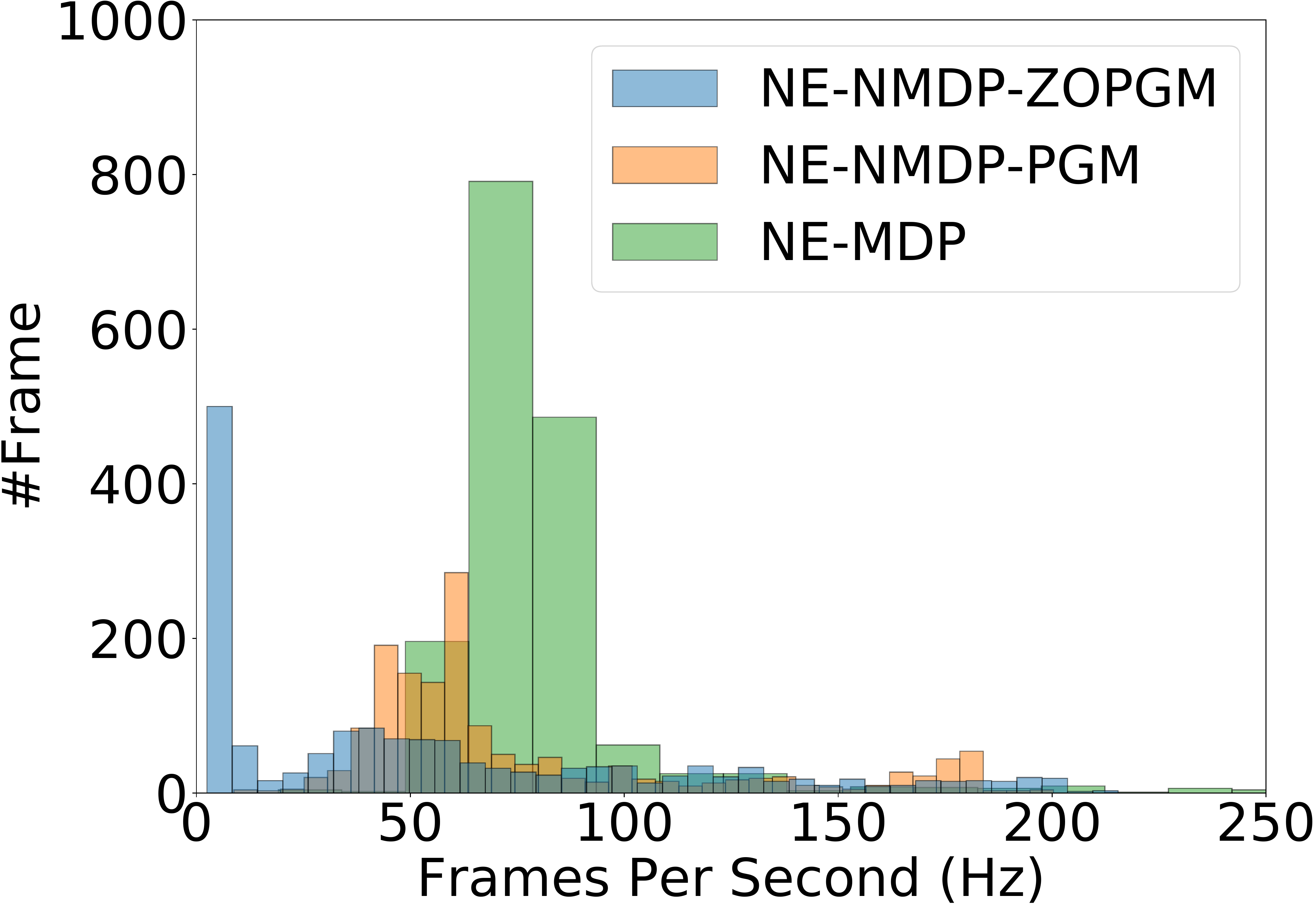}\put(-25,20){(a)}&
\includegraphics[width=0.195\textwidth]{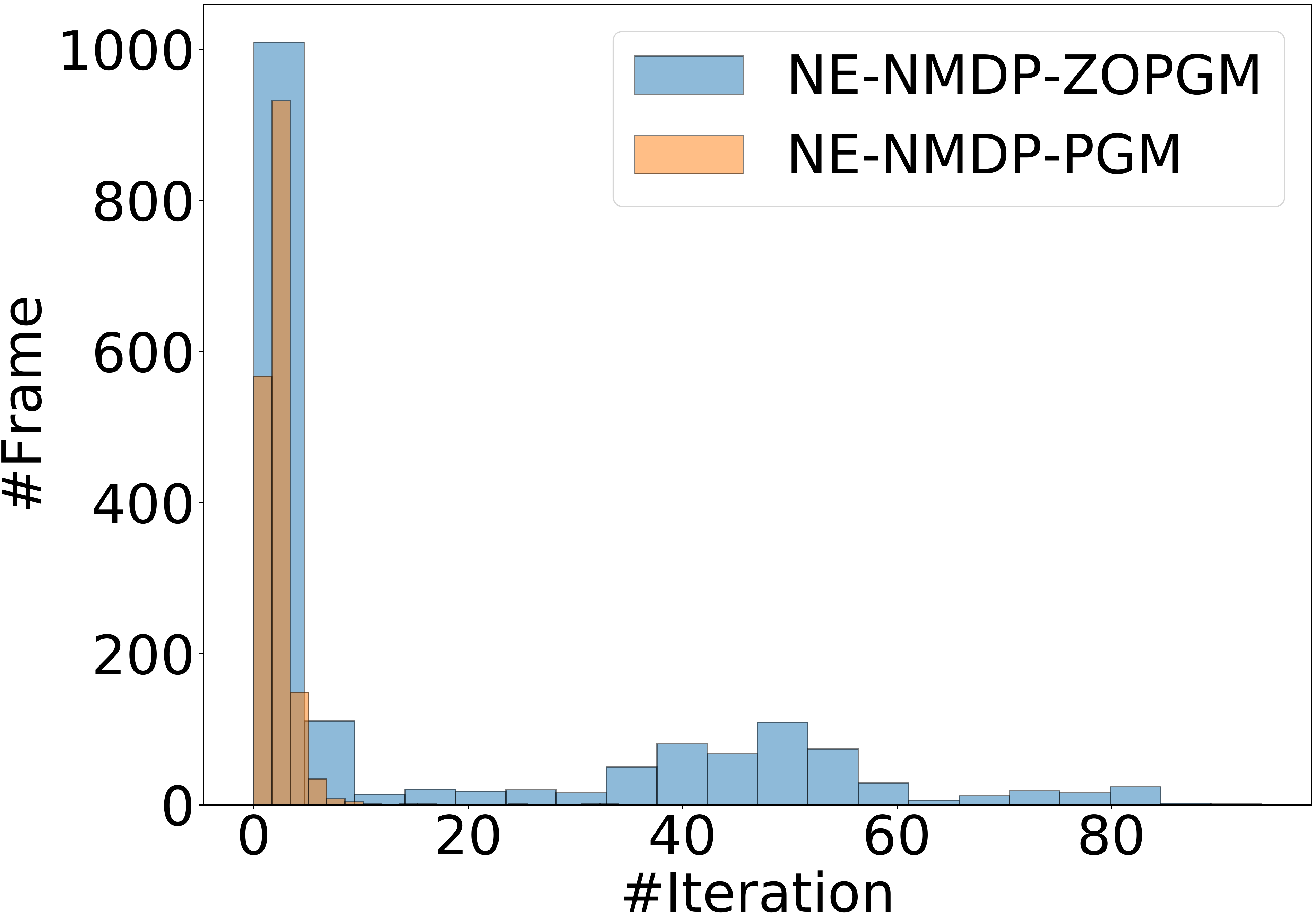}\put(-25,20){(b)}&
\includegraphics[width=0.19\textwidth]{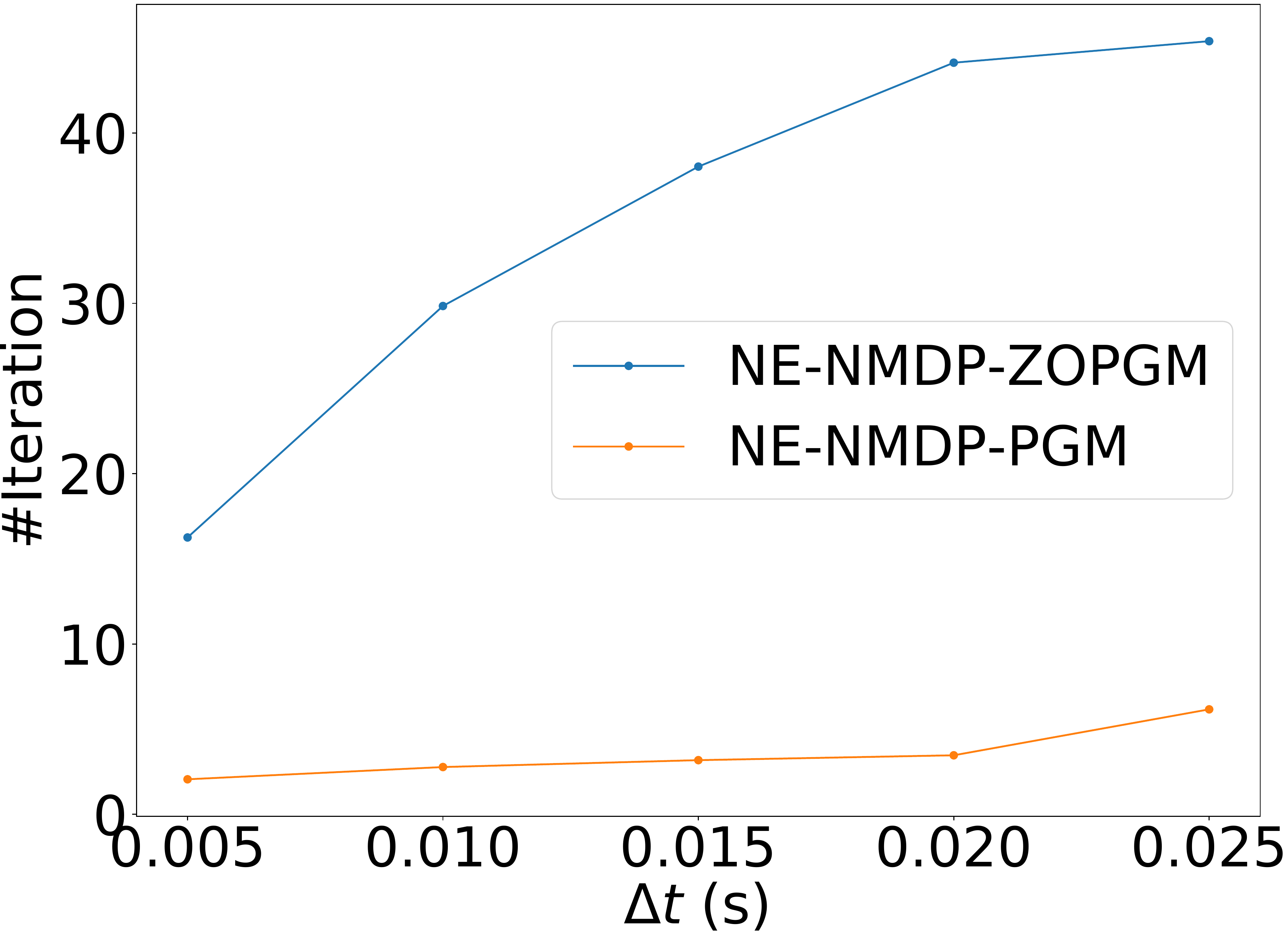}\put(-25,20){(c)}&
\multicolumn{2}{c}{\includegraphics[width=0.37\textwidth]{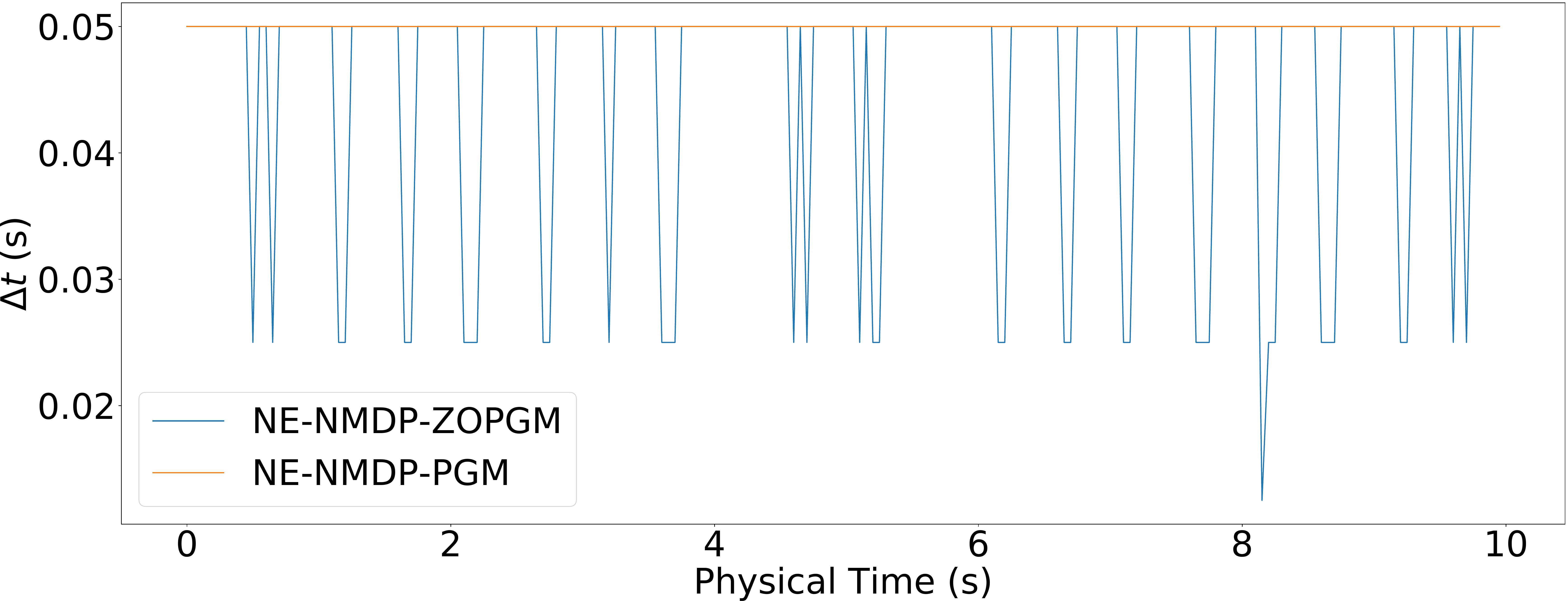}\put(-25,20){(d)}}
\end{tabu}
\vspace{-8px}
\caption{\label{fig:cost} \small{A performance comparison of different methods. (a): Histogram of framerates (b): Histogram of number of outer iterations per frame (c): Number of outer iterations against $\Delta t$ (d): Adaptive timestep size}}
\vspace{-10px}
\end{figure*}
\TE{Stability Under Large Timestep Sizes:} We track a robot jumping trajectory that uses symmetric poses. Since the Robosimian's body shape is also symmetric, the torso in the tracked trajectory should have zero tilt angles from the vertical axis, which is our groundtruth. In \prettyref{fig:jump} (ab), we plot the torso's tile angle predicted by NE-MDP and NE-NMDP-PGM under different timestep sizes. NE-MDP is only stable when $\Delta t\leq7$ms and the simulator explodes under larger $\Delta t$. Even when $\Delta t\leq7$ms, the tile angle suffers from severe oscillation. In comparison, NE-NMDP-PGM is stable when $\Delta t$ increases from $5$ms to $50$ms and the predicted tile angle oscillation is relatively small. We further plot the tile angle sequence predicted by the other three variants of PGM in \prettyref{fig:jump} (cde), the oscillations are consistently small. When using PGM with $\Delta t\leq50$ms, we have never experienced any divergence behaviors so \prettyref{alg:adaptive} is never needed. But divergence happens in ZOPGM and we observe larger tile angle oscillation in \prettyref{fig:jump} (ce).

\TE{Consistent Prediction:} The tracked robot jumping trajectory lasts for 10 seconds with 5 repeated jumping behaviors. As a result, the expected torso height should be a periodic function. In \prettyref{fig:height}, we plot the torso height trajectory predicted using the five methods. With $\Delta t\leq7$ms, the trajectories generated by NE-MDP are suffering from relatively large variations. While all four PGM variants can consistently predict a periodic function when $\Delta t$ increases from $5$ms to $50$ms. A similar result is observed in the walking trajectory. As shown in \prettyref{fig:distance}, the groundtruth walking distance is a linear function of time. The trajectories predicted using NE-MDP exhibit a large variations with different timestep sizes. The consistency of NE-NMDP-(ZO)PGM are much better and that of PBD-NMDP-(ZO)PGM are the best. Compared with NE-NMDP-(ZO)PGM, the additional stability and consistency of PBD-NMDP-(ZO)PGM are presumably due to the fact that Euclidean space discretization is more accurate than that in the configuration space \cite{10.1007/978-3-030-44051-0_39}.

\TE{Computational Cost:} We summarize the computational performance by collecting and analyzing all the timesteps in the four trajectories of \prettyref{fig:results} simulated at $\Delta t=50$ms. \prettyref{fig:cost} (a) profiles the instantaneous framerate, of which MDP is the most efficient involving a single QP solve. PGM does not incur a significant sacrifice in framerate while ZOPGM is significantly slower. \prettyref{fig:cost} (b) profiles the number of outer iterations of \prettyref{alg:PGM} until convergence and \prettyref{fig:cost} (c) plots the average number of outer iterations against $\Delta t$. These figures show that ZOPGM is an approximately one order of magnitude slower than PGM, but ZOPGM provides the extra convenience that analytic derivatives of $v_x(\theta)$ is not needed. Finally, \prettyref{fig:cost} (d) shows the smallest timestep size chosen by \prettyref{alg:adaptive} for PGM and ZOPGM. At $\Delta t=50$ms, PGM is always convergent. ZOPGM is convergent for most timesteps, but $\Delta t$ needs to get down to $25$ms during some critical time instances (e.g. when robot changes contact state).


%% file: conclusion.tex
\section{\label{sec:conclusion}Conclusion \& Discussion}
We present NMDP, a backward-Euler time integration scheme for articulated bodies under generalized contact models. The key to our formulation is the representation of contact forces as a convex combination of vertices of the feasible contact wrench space. To model generalized contact models, we assume that these vertices are dependent on the robot pose. Following the idea of backward-Euler time integrator, we discretize both the articulated body dynamics and the vertices of contact wrench spaces at the next time instance instead of the current one. We solve the constrained optimization using a projected gradient method. Our analysis proves that NMDP has guaranteed convergence under small timestep sizes and a robust simulator can be built using an adaptive timestep size control algorithm. Empirically, we show that this scheme has better stability under large timestep sizes and consistency over varying timestep sizes.

We plan to address several limitations in our future work. First, the main idea of NMDP is not limited to articulated bodies and can be extended to high-dimensional dynammic systems such as deformable objects. However, it requires exact inversion of the constraint Jacobian, which is impractical in high-dimensional scenarios. Second, we assume a small set of known contact points. When new contacts are detected, we do not handle them in our implementation. Third, the generalized contact model needs to be smooth and cannot provide collision-free guarantees. Finally, we plan to evaluate the performance of optimization-based motion planners and controllers using NMDP as the underlying integration model. 

%% file: appendix.tex
\section{\label{sec:appendix}Appendix: Proof}
We define some convenient shorthand notations. 
\begin{align*}
G_\alpha(\theta,w)=&\frac{1}{\alpha}A(\theta)+B(\theta,w)    \\
A(\theta)=&\int_{x\in\mathcal{R}}\frac{\rho}{\Delta t^2}
\FPPN{X}{\theta}^T(X(x,\theta)-X(x,\theta_{-}))dx \\
B(\theta,w)=&\int_{x\in\mathcal{R}}\frac{\rho}{\Delta t^2}
\FPPN{X}{\theta}^T(X(x,\theta_{--})-X(x,\theta_{-}))dx- \\
&\sum_{x\in\mathcal{C}}\FPPN{X}{\theta}(x,\theta)^Tv_x(\theta)w_x-\tau.
\end{align*}
Note that all variables or functions that depend on $\alpha$ will be labeled using a subscript, and vice versa.
\input{continuousConvexity.tex}
\input{discreteProjection.tex}
\input{discreteProjectionZO.tex}
\input{convergence.tex}

%% file: continuousConvexity.tex
\subsection{Continuous Manifold Projection}
We begin our analysis by showing that $\theta$ satisfying $G_\alpha(\theta,w)=0$ can be arbitrarily close to $\theta_{-}$ by using sufficiently small $\alpha$. This can be proved by analyzing the following Lyapunov candidate:
\begin{align*}
V_\alpha(\theta,w)=\|G_\alpha(\theta,w)\|^2=\frac{1}{\alpha^2}A^TA+\frac{2}{\alpha}A^TB+B^TB.
\end{align*}
We prove basic properties of boundedness and convexity:
\begin{lemma}
\label{Lem:basic}
Assuming \prettyref{ass:XSmooth},\prettyref{ass:SigmaMin},\prettyref{ass:VSmooth}, there exists an $r>0$, such that for all feasible $w$, the following properties hold within $\mathcal{B}(\theta_{-},r)$:
\begin{enumerate}
\item $\sigma_{min}(\FPPN{G_\alpha}{\theta})\geq\frac{\sigma_X}{2\alpha}-\sigma_B$.
\item $\sigma_{min}(\FPPTN{V_\alpha}{\theta})\geq
\frac{\sigma_X^2}{2\alpha^2}-\frac{2}{\alpha}\sigma_{AB}-\sigma_{BB}$.
\item $\|\FPPN{G_\alpha}{\theta}\|\leq \frac{1}{\alpha}M_A^1+M_B^1$.
\item $\|\FPPTN{G_\alpha}{\theta}\|\leq \frac{1}{\alpha}M_A^2+M_B^2$.
\item $\|\FPPN{V_\alpha}{\theta}\|\leq \frac{1}{\alpha^2}M_{AA}^1+\frac{1}{\alpha}M_{AB}^1+M_{BB}^1$.
\item $\|\FPPTN{V_\alpha}{\theta}\|\leq \frac{1}{\alpha^2}M_{AA}^2+\frac{1}{\alpha}M_{AB}^2+M_{BB}^2$.
\item $\|\FPPTTTN{V_\alpha}{\theta}\|\leq \frac{1}{\alpha^2}M_{AA}^3+\frac{1}{\alpha}M_{AB}^3+M_{BB}^3$.
\item $\|\FPPN{V_\alpha}{\theta}(\theta_{-},w)\|\leq \frac{1}{\alpha}M_{AG}+M_{BG}$.
\item $\|G_\alpha(\theta_{-},w)\|\leq M_G$.
\item $\|B(\theta,w)-B(\theta',w')\|\leq M_{\Delta B}$.
\end{enumerate}
\end{lemma}
\begin{proof}
These are immediate results of continuity of functions and singular values of matrices. 
\noindent{\TE{1):}}
\begin{small}
\begin{align*}
&\sigma_{min}(\FPPN{G_\alpha}{\theta})  \\
=&\sigma_{min}(\frac{1}{\alpha}\FPPN{A}{\theta}+\FPPN{B}{\theta})   \\
\geq&\sigma_{min}(\frac{1}{\alpha}\FPPN{A}{\theta})-\sigma_{max}(\FPPN{B}{\theta})
\BECAUSE{\cite[Equation~3.3.17]{horn1994topics}}  \\
\geq&\sigma_{min}(\frac{1}{2\alpha}
\FPPN{A}{\theta}(\theta_{-}))-\sigma_{max}(\FPPN{B}{\theta})
\BECAUSE{\cite[Corollary~8.6.2]{GoluVanl96},\prettyref{ass:XSmooth}}  \\
\geq&\frac{\sigma_X}{2\alpha}-\sigma_{max}(\FPPN{B}{\theta})
\BECAUSE{\prettyref{ass:SigmaMin}}  \\
\geq&\frac{\sigma_X}{2\alpha}-\sigma_B.
\BECAUSE{\cite[Corollary~8.6.2]{GoluVanl96},\prettyref{ass:XSmooth},\prettyref{ass:VSmooth}}
\end{align*}
\end{small}
\noindent{\TE{2):}}
\begin{small}
\begin{align*}
&\sigma_{min}(\FPPTN{V_\alpha}{\theta}) \\
=&\sigma_{min}(\frac{1}{\alpha^2}\FPPTN{\left[A^TA\right]}{\theta}+
\frac{2}{\alpha}\FPPTN{\left[A^TB\right]}{\theta}+
\FPPTN{\left[B^TB\right]}{\theta})  \\
\geq&\frac{1}{\alpha^2}\sigma_{min}(\FPPTN{\left[A^TA\right]}{\theta})-
\sigma_{max}(\frac{2}{\alpha}\FPPTN{\left[A^TB\right]}{\theta})-    \\
&\sigma_{max}(\FPPTN{\left[B^TB\right]}{\theta}))
\BECAUSE{\cite[Equation~3.3.17]{horn1994topics}}    \\
\geq&\frac{1}{2\alpha^2}\sigma_{min}(\FPPTN{\left[A^TA\right]}{\theta}(\theta_{-}))-
\sigma_{max}(\frac{2}{\alpha}\FPPTN{\left[A^TB\right]}{\theta})-    \\
&\sigma_{max}(\FPPTN{\left[B^TB\right]}{\theta}))
\BECAUSE{\cite[Corollary~8.6.2]{GoluVanl96},\prettyref{ass:XSmooth}}  \\
\geq&\frac{\sigma_X^2}{2\alpha^2}-
\sigma_{max}(\frac{2}{\alpha}\FPPTN{\left[A^TB\right]}{\theta})-
\sigma_{max}(\FPPTN{\left[B^TB\right]}{\theta})\BECAUSE{\prettyref{ass:SigmaMin}}  \\
\geq&\frac{\sigma_X^2}{2\alpha^2}-\frac{2}{\alpha}\sigma_{AB}-\sigma_{BB}.
\BECAUSE{\cite[Corollary~8.6.2]{GoluVanl96},\prettyref{ass:XSmooth},\prettyref{ass:VSmooth}}  \\
\end{align*}
\end{small}
\noindent{\TE{3),4):}}
\begin{align*}
&\|\FPPSN{G_\alpha}{\theta}\|
=\|\frac{1}{\alpha}\FPPSN{A}{\theta}+\FPPSN{B}{\theta}\|    \\
\leq&\frac{1}{\alpha}\|\FPPSN{A}{\theta}\|+\|\FPPSN{B}{\theta}\|\BECAUSE{triangle}  \\
\leq&\frac{1}{\alpha}M_A^*+M_B^*.
\BECAUSE{\prettyref{ass:XSmooth},\prettyref{ass:VSmooth}}
\end{align*}
\noindent{\TE{5),6),7):}}
\begin{align*}
&\|\FPPSN{V_\alpha}{\theta}\|
=\|\frac{1}{\alpha^2}\FPPSN{\left[A^TA\right]}{\theta}+
\frac{2}{\alpha}\FPPSN{\left[A^TB\right]}{\theta}+
\FPPSN{\left[B^TB\right]}{\theta}\| \\
\leq&\frac{1}{\alpha^2}\|\FPPSN{\left[A^TA\right]}{\theta}\|+
\frac{2}{\alpha}\|\FPPSN{\left[A^TB\right]}{\theta}\|+
\|\FPPSN{\left[B^TB\right]}{\theta}\|   \\
&\BECAUSE{triangle}  \\
\leq&\frac{1}{\alpha^2}M_{AA}^*+\frac{2}{\alpha}M_{AB}^*+M_{BB}^*.
\BECAUSE{\prettyref{ass:XSmooth},\prettyref{ass:VSmooth}}
\end{align*}
\noindent{\TE{8):}}
\begin{align*}
&\|\FPPN{V_\alpha}{\theta}(\theta_{-},w)\|
=\|\FPPN{G_\alpha}{\theta}(\theta_{-})^TG_\alpha(\theta_{-},w)\|    \\
=&\|(\frac{1}{\alpha}\FPPN{A}{\theta}(\theta_{-})+\FPPN{B}{\theta}(\theta_{-}))^TB(\theta_{-},w)\|   \\
\leq&\frac{1}{\alpha}\|\FPPN{A}{\theta}(\theta_{-})^TB(\theta_{-},w)\|+
\|\FPPN{B}{\theta}(\theta_{-})^TB(\theta_{-},w)\|   \\
&\BECAUSE{triangle}   \\
\leq&\frac{1}{\alpha}M_{AG}+M_{BG}.\BECAUSE{\prettyref{ass:XSmooth},\prettyref{ass:VSmooth}}
\end{align*}
\noindent{\TE{9):}}
\begin{small}
\begin{align*}
&\|G_\alpha(\theta_{-},w)\|
\leq\frac{1}{\alpha}\|A(\theta_{-})\|+\|B(\theta_{-},w)\|
=\|B(\theta_{-},w)\|\leq M_G.
\end{align*}
\end{small}
\noindent{\TE{10):}} by \prettyref{ass:XSmooth},\prettyref{ass:VSmooth} and boundedness of feasible $w$.
\end{proof} 
Also in the region $\mathcal{B}(\theta_{-},r)$, we have that \prettyref{eq:PROJ}, \prettyref{eq:QP}, and \prettyref{alg:PGM} are well-defined because \prettyref{Lem:basic} implies that $\FPPN{G_\alpha}{\theta}$ is invertible:
\begin{corollary}
\label{Cor:welldefined}
Assuming \prettyref{ass:XSmooth},\prettyref{ass:SigmaMin},\prettyref{ass:VSmooth}, for $r$ chosen as in \prettyref{Lem:basic}, \prettyref{alg:PGM} is well-defined as long as $\{\theta^k\}\subset\mathcal{B}(\theta_{-},r)$.
\end{corollary}

We can easily show that in a small vicinity of $\theta_{-}$, manifold projection always has a solution:
\begin{theorem}[Continuous Convergence]
\label{lem:CConv}
Assuming \prettyref{ass:XSmooth},\prettyref{ass:SigmaMin},\prettyref{ass:VSmooth}, for $r$ chosen as in \prettyref{Lem:basic}, there exists $\alpha_1>0$, such that for any $\alpha\leq\alpha_1$ and feasible $w$, the solution to $G_\alpha(\theta,w)=0$ computed using the negative gradient flow $\dot{\theta}=-\FPPN{V_\alpha}{\theta}$ from initial guess $\theta_{-}$ is within $\mathcal{B}(\theta_{-},r)$.
\end{theorem}
\begin{proof}
\TE{Solution Boundedness:} Consider $V_\alpha(\theta,w)-V_\alpha(\theta_{-},w)$, with $\theta\in\partial\mathcal{B}(\theta_{-},r)$. We have strong convexity along line-segment connecting $\theta,\theta^+$ and by \prettyref{Lem:basic}:
\begin{small}
\begin{align*}
\label{eq:Bound}
&V_\alpha(\theta,w)-V_\alpha(\theta_{-},w)   \\
\geq&\FPPN{V_\alpha}{\theta}(\theta_{-},w)^T(\theta-\theta_{-})+
(\frac{\sigma_X^2}{2\alpha^2}-
\frac{2}{\alpha}\sigma_{AB}-\sigma_{BB})
\|\theta-\theta_{-}\|^2   \\
\geq&r^2(\frac{\sigma_X^2}{2\alpha^2}-
\frac{2}{\alpha}\sigma_{AB}-\sigma_{BB})-r(\frac{1}{\alpha}M_{AG}+M_{BG})
=\mathcal{O}(\frac{1}{\alpha^2}).
\end{align*}
\end{small}
If we choose small enough $\alpha_1$ such that the last equation is larger than zero for all $\alpha\leq\alpha_1$, then $V_\alpha$ becomes a Lyapunov function of the gradient flow, restricting the solution to $\mathcal{B}(\theta_{-},r)$ due to Nagumo's Theorem \cite{Aubin2009}. \TE{Convergence:} We still need to show the gradient flow converges to a solution of $G_\alpha(\theta,w)=0$, which is trivial due to the following inequality:
\begin{align*}
\dot{V}_\alpha=&-\|\FPPN{V_\alpha}{\theta}\|^2   =-\|G_\alpha^T\FPPN{G_\alpha}{\theta}\FPPN{G_\alpha}{\theta}^TG_\alpha\|    \\
\leq&-(\frac{\sigma_X}{2\alpha}-\sigma_B)^2\|G_\alpha\|^2
=-(\frac{\sigma_X}{2\alpha}-\sigma_B)^2V_\alpha.
\end{align*}
We can again choose small enough $\alpha_1$ such that the coefficient of $V_\alpha$ in the last inequality is smaller than zero. 
\end{proof}
\prettyref{lem:CConv} implies that, with sufficiently small $\alpha$, PGM will always generate a sequence that is within $\mathcal{B}(\theta_{-},r)$ if manifold projection is solved by exactly time-integrating the negative gradient flow.

%% file: discreteProjection.tex
\subsection{Discrete First-Order Manifold Projection}
In this section, we go beyond \prettyref{lem:CConv} and analyze the practical discrete manifold projection algorithm, i.e. \prettyref{eq:PROJ}. Let's define the shorthand notation:
\begin{align*}
\theta^+=\theta-\FPPN{G_\alpha}{\theta}^{-1}G_\alpha(\theta,w).
\end{align*}
We first analyze the property of $\theta^+$, i.e. one iteration of manifold projection. We show that the relative change $\theta^+-\theta$ is bounded:
\begin{lemma}
\label{Lem:oneStepBounded}
Assuming \prettyref{ass:XSmooth},\prettyref{ass:SigmaMin},\prettyref{ass:VSmooth}, $V_\alpha(\theta,w)\leq M_V$ where $M_V$ is some $\alpha$-independent constant. For $r$ chosen as in \prettyref{Lem:basic}, there exists $\alpha_2>0$, such that for any $\alpha\leq\alpha_2$ and feasible $w$, $\|\theta^+-\theta\|\leq r/2$ if $\theta^+,\theta\in B(\theta_{-},r)$.
\end{lemma}
\begin{proof}
If we choose $\alpha_2\leq\alpha_1$:
\begin{align*}
&\|\theta^+-\theta\|=\|\FPPN{G_\alpha}{\theta}^{-1}G_\alpha\| \\
\leq&\|G_\alpha\|(\frac{\sigma_X}{2\alpha}-\sigma_B)^{-1}
\leq\sqrt{M_V}(\frac{\sigma_X}{2\alpha}-\sigma_B)^{-1}\leq r/2.
\end{align*}
The last inequality above holds by choosing small enough $\alpha_2$ and the proof is complete.
\end{proof}
We then show that it is possible to choose sufficiently small $\alpha$ such that the absolute norm of $\theta^+$ is bounded:
\begin{lemma}
\label{Lem:discreteInvariance}
Assuming \prettyref{ass:XSmooth},\prettyref{ass:SigmaMin},\prettyref{ass:VSmooth}, $V_\alpha(\theta,w)\leq M_V$ where $M_V$ is some $\alpha$-independent constant. For any $\beta\in(0,1]$ and $r$ chosen as in \prettyref{Lem:basic}, there exists $\alpha_3(\beta)>0$, such that for any $\alpha\leq\alpha_3(\beta)$ and feasible $w$, the following properties hold:
\begin{enumerate}
\item If $\theta\in\mathcal{B}(\theta_{-},r\beta/2)$, then $\theta^+\in\mathcal{B}(\theta_{-},r\beta/2)$.
\item $V_\alpha(\theta^+,w)\leq 3V_\alpha(\theta,w)/4$.
\end{enumerate}
\end{lemma}
\begin{proof}
\TE{Monotonic Reduction:} As long as $\alpha\leq\alpha_3(\beta)\leq\alpha_2$, we have strong convexity over the line-segment connecting $\theta,\theta^+$. This is because $\theta\in\mathcal{B}(\theta_{-},r)$, $\theta^+\in\mathcal{B}(\theta_{-},r(1+\beta)/2)\subset\mathcal{B}(\theta_{-},r)$, and \prettyref{Lem:basic} implies strong convexity. By the Taylor's expansion theorem, we have:
\begin{small}
\begin{align*}
&V_\alpha(\theta^+,w)   \\
=&V_\alpha(\theta,w)+\FPPN{V_\alpha}{\theta}^T(\theta^+-\theta)+    \\
&\frac{1}{2}(\theta^+-\theta)^T\FPPTN{V_\alpha}{\theta}(\theta^+-\theta)+R_V    \\
=&\frac{1}{2}(\theta^+-\theta)^T\FPPTN{V_\alpha}{\theta}(\theta^+-\theta)+R_V  \\
=&\frac{1}{2}(\theta^+-\theta)^T\sum_i\FPPTN{[G_\alpha]_i}{\theta}[G_\alpha]_i(\theta^+-\theta)+    R_V+\frac{1}{2}V_\alpha(\theta,w) \\
\leq&\frac{1}{2}(\frac{1}{\alpha}M_A^2+M_B^2)\|\theta^+-\theta\|^2\|G_\alpha\|+\frac{1}{2}V_\alpha(\theta,w)+    \\
&\frac{1}{6}(\frac{1}{\alpha^2}M_{AA}^3+\frac{1}{\alpha}M_{AB}^3+M_{BB}^3)\|\theta^+-\theta\|^3  \\
\leq&\frac{1}{2}(\frac{1}{\alpha}M_A^2+M_B^2)(\frac{\sigma_X}{2\alpha}-\sigma_B)^{-2}V_\alpha(\theta,w)^{1.5}+
\frac{1}{2}V_\alpha(\theta,w)+    \\
&\frac{1}{6}(\frac{1}{\alpha^2}M_{AA}^3+\frac{1}{\alpha}M_{AB}^3+M_{BB}^3)
(\frac{\sigma_X}{2\alpha}-\sigma_B)^{-3}V_\alpha(\theta,w)^{1.5} \\
\leq&\frac{1}{2}(\frac{1}{\alpha}M_A^2+M_B^2)(\frac{\sigma_X}{2\alpha}-\sigma_B)^{-2}\sqrt{M_V}V_\alpha(\theta,w)+\frac{1}{2}V_\alpha(\theta,w)+    \\
&\frac{1}{6}(\frac{1}{\alpha^2}M_{AA}^3+\frac{1}{\alpha}M_{AB}^3+M_{BB}^3)
(\frac{\sigma_X}{2\alpha}-\sigma_B)^{-3}\sqrt{M_V}V_\alpha(\theta,w)  \\
=&(\frac{1}{2}+\mathcal{O}(\alpha))V_\alpha(\theta,w),
\end{align*}
\end{small}
where $R_V$ is the residual term of the Taylor's expansion. We can choose $\alpha_3(\beta)$ small enough so that the coefficient of $V_\alpha(\theta,w)$ in the last inequality is smaller than $3/4$ for all $\alpha\leq\alpha_3(\beta)$. \TE{Sequence Boundedness:} We prove $\theta^+\in\mathcal{B}(\theta_{-},r\beta/2)$ by contradiction. If we pick $\alpha_3(\beta)$ to ensure monotonic reduction, then $\theta^+\in\mathcal{B}(\theta_{-},r(\beta+1)/2)$ (\prettyref{Lem:oneStepBounded}). If $\theta^+\notin\mathcal{B}(\theta_{-},r\beta/2)$ then $\theta^+\in\mathcal{B}(\theta_{-},r(\beta+1)/2)-\mathcal{B}(\theta_{-},r\beta/2)$. Since strong convexity holds along line segment connecting $\theta^+,\theta_{-}$, we have:
\begin{small}
\begin{align*}
&V_\alpha(\theta^+,w)-V_\alpha(\theta_{-},w)   \\
\geq&\FPPN{V_\alpha}{\theta}(\theta_{-},w)^T(\theta^+-\theta_{-})+
(\frac{\sigma_X^2}{2\alpha^2}-
\frac{2}{\alpha}\sigma_{AB}-\sigma_{BB})\|\theta^+-\theta_{-}\|^2   \\
\geq&\frac{r^2\beta^2}{4}(\frac{\sigma_X^2}{2\alpha^2}-
\frac{2}{\alpha}\sigma_{AB}-\sigma_{BB})-\frac{r(\beta+1)}{2}(\frac{1}{\alpha}M_{AG}+M_{BG})    \\
\geq&M_V.
\end{align*}
\end{small}
Similarly, we can choose $\alpha_3(\beta)$ small enough so that the last inequality holds for all $\alpha\leq\alpha_3(\beta)$. This violates the fact that $V_\alpha(\theta^+,w)\leq 3V_\alpha(\theta_{-},w)/4\leq 3M_V/4$.
\end{proof}
Now we have the desired convergence result as long as $V_\alpha(\theta,w)$ is upper bounded by some $M_V$, which is proved and the following theorem:
\begin{theorem}[Discrete Convergence]
\label{Thm:DConv}
Assuming \prettyref{ass:XSmooth},\prettyref{ass:SigmaMin},\prettyref{ass:VSmooth}, there exists some $\alpha$-independent constant $M_V$, such that for $r$ chosen as in \prettyref{Lem:basic}, $\alpha_3(\beta)$ chosen as in \prettyref{Lem:discreteInvariance}, \prettyref{alg:PGM} will generate a sequence $\{\theta^k,w^k\}$ with the following properties for all $\alpha\leq\alpha_3(\beta)$:
\begin{enumerate}
\item $\{V_\alpha(\theta^k,w^k)\}$ is upper bounded by $M_V$.
\item $\{\theta^k\}$ is within $\mathcal{B}(\theta_{-},r\beta/2)$.
\item Each manifold projection is convergent at a rate of at least $3/4$.
\end{enumerate}
\end{theorem}
\begin{proof}
We can proceed the proof by setting $M_V=M_G^2+2M_GM_{\Delta B}+M_{\Delta B}^2$. \TE{Induction:} We prove by induction. The initial guess $V_\alpha(\theta^0,w^0)\leq M_V$ and $\theta^0\in\mathcal{B}(\theta_{-},r\beta/2)$. If $V_\alpha(\theta^{k-1},w^{k-1})\leq M_V$ and $\theta^{k-1}\in\mathcal{B}(\theta_{-},r\beta/2)$, then \prettyref{alg:PGM} will first update $w^{k-1}$ to $w^k$ and then update $\theta^{k-1}$ to $\theta^k$ via manifold projection. We use triangle inequality to bound the change of $V_\alpha$ due to the update of $w$ as follows:
\begin{small}
\begin{align*}
&V_\alpha(\theta^{k-1},w^k)-V_\alpha(\theta^{k-1},w^{k-1})   \\
=&(G_\alpha(\theta^{k-1},w^k)+G_\alpha(\theta^{k-1},w^{k-1}))^T \\
&(G_\alpha(\theta^{k-1},w^k)-G_\alpha(\theta^{k-1},w^{k-1})) \\
\leq&\|(\frac{2}{\alpha}A(\theta^{k-1})+B(\theta^{k-1},w^{k-1})+B(\theta^{k-1},w^k))\|   \\
&\|(B(\theta^{k-1},w^{k-1})-B(\theta^{k-1},w^k))\|\BECAUSE{triangle}  \\
=&\|\frac{2}{\alpha}A(\theta^{k-1})+2B(\theta^{k-1},w^{k-1})-B(\theta^{k-1},w^{k-1})+B(\theta^{k-1},w^k)\|   \\
&\|(B(\theta^{k-1},w^{k-1})-B(\theta^{k-1},w^k))\|  \\
\leq&2\|G_\alpha(\theta^{k-1},w^{k-1})\|
\|B(\theta^{k-1},w^{k-1})-B(\theta^{k-1},w^k)\|+   \\
&\|(B(\theta^{k-1},w^{k-1})-B(\theta^{k-1},w^k))\|^2\BECAUSE{triangle}    \\
\leq&2M_GM_{\Delta B}+M_{\Delta B}^2,
\end{align*}
\end{small}
so we have:
\begin{small}
\begin{align*}
V_\alpha(\theta^{k-1},w^k)\leq& V_\alpha(\theta^{k-1},w^{k-1})+2M_GM_{\Delta B}+M_{\Delta B}^2   \\
\leq& V_\alpha(\theta_{-},w^{k-1})+2M_GM_{\Delta B}+M_{\Delta B}^2   \\
\leq& M_G^2+2M_GM_{\Delta B}+M_{\Delta B}^2=M_V,
\end{align*}
\end{small}
where in the last inequality we have used the fact that $\|G_\alpha(\theta^{k-1},w^{k-1})\|\leq\|G_\alpha(\theta_{-},w^{k-1})\|$. Next, we enter the phase of manifold projection. By \prettyref{Lem:discreteInvariance}, each step of manifold projection is reducing $V_\alpha$ at a rate of at least $3/4$ and staying inside $\mathcal{B}(\theta_{-},r\beta/2)$. The proof is complete. \TE{Solution Uniqueness:} In the above proof we have used the fact that $\|G_\alpha(\theta^{k-1},w^{k-1})\|\leq\|G_\alpha(\theta_{-},w^{k-1})\|$. If we always start manifold projection from initial guess $\theta_{-}$, then this is true due to \prettyref{Lem:discreteInvariance}. But if we start from any where in $\mathcal{B}(\theta_{-},r\beta/2)$, this is also true because (1) \prettyref{Lem:discreteInvariance} guarantees convergence; (2) Due to strong convexity \prettyref{Lem:basic}, there is a unique solution to $V_\alpha=0$ inside $\mathcal{B}(\theta_{-},r\beta/2)$. 
\end{proof}
This result shows that all variables in \prettyref{alg:PGM} are bounded and the manifold project substep is convergence.

%% file: discreteProjectionZO.tex
\subsection{Discrete Zeroth-Order Manifold Projection}
The zeroth-order manifold projection can be proved to be convergent in an almost identical manner as that of \prettyref{Thm:DConv} except for \prettyref{Lem:discreteInvariance}. We use the following shorthand notations:
\begin{align*}
\theta^+=&\theta-\FPPN{\bar{G}_\alpha}{\theta}^{-1}G_\alpha(\theta,w)  \\
\Delta G=&\FPPN{G_\alpha}{\theta}-\FPPN{\bar{G}_\alpha}{\theta}   \\
\Delta G_\alpha^{-1}=&\FPPN{G_\alpha}{\theta}^{-1}-\FPPN{\bar{G}_\alpha}{\theta}^{-1}   \\
\FPPN{\bar{G}_\alpha}{\theta}^{-1}=&\FPPN{G_\alpha}{\theta}^{-1}-\Delta G_\alpha^{-1}   \\
\FPPN{G_\alpha}{\theta}\Delta G_\alpha^{-1}=&
\E{I}-\FPPN{G_\alpha}{\theta}\FPPN{\bar{G}_\alpha}{\theta}^{-1}
=-\Delta G\FPPN{\bar{G}_\alpha}{\theta}^{-1}.
\end{align*}
The key to our proof is the fact that $\FPPN{G_\alpha}{\theta}$ and $\FPPN{\bar{G}_\alpha}{\theta}$ differs by an $\alpha$-independent term. To proceed, we need to add the following results to \prettyref{Lem:basic}:
\begin{lemma}
\label{Lem:basicZO}
Assuming \prettyref{ass:XSmooth},\prettyref{ass:SigmaMin},\prettyref{ass:VSmooth},\prettyref{ass:GwFullrank}, there exists an $r>0$, such that for all feasible $w$, both \prettyref{Lem:basic} and the following properties hold within $\mathcal{B}(\theta_{-},r)$:
\begin{enumerate}
\item $\sigma_{min}(\FPPN{G}{w}\FPPN{G}{w}^T)\geq\sigma_w>0$.\BECAUSE{\prettyref{ass:GwFullrank}}
\item $\|\FPPN{G}{w}\FPPN{G}{w}^T\|\leq M_w$.
\item $\sigma_{min}(\FPPN{\bar{G}_\alpha}{\theta})\geq
\frac{\sigma_X}{2\alpha}-\bar{\sigma}_B$.
\item $\sigma_{max}(\FPPN{G_\alpha}{\theta}\Delta G_\alpha^{-1})\leq
M_{\Delta G}(\frac{\sigma_X}{2\alpha}-\bar{\sigma}_B)^{-1}$.
\item $\sigma_{max}(\Delta G_\alpha^{-1})\leq
\sigma_{max}(\FPPN{G_\alpha}{\theta}^{-1})
\sigma_{max}(\FPPN{G_\alpha}{\theta}\Delta G_\alpha^{-1})
\leq M_{\Delta G}(\frac{\sigma_X}{2\alpha}-\bar{\sigma}_B)^{-1}
(\frac{\sigma_X}{2\alpha}-\sigma_B)^{-1}$.
\end{enumerate}
\end{lemma}
We omit the proof which is similar to \prettyref{Lem:basic}. Next, we prove the zeroth-order variant of \prettyref{Lem:discreteInvariance} below:
\begin{lemma}
\label{Lem:discreteInvarianceZO}
Assuming \prettyref{ass:XSmooth},\prettyref{ass:SigmaMin},\prettyref{ass:VSmooth}, $V_\alpha(\theta,w)\leq M_V$ where $M_V$ is some $\alpha$-independent constant. For any $\beta\in(0,1]$ and $r$ chosen as in \prettyref{Lem:basic}, there exists $\alpha_3(\beta)>0$, such that for any $\alpha\leq\alpha_3(\beta)$ and feasible $w$, the following properties hold:
\begin{enumerate}
\item If $\theta\in\mathcal{B}(\theta_{-},r\beta/2)$, then $\theta^+\in\mathcal{B}(\theta_{-},r\beta/2)$.
\item $V_\alpha(\theta^+,w)\leq 3V_\alpha(\theta,w)/4$.
\end{enumerate}
\end{lemma}
\begin{proof}
\TE{Monotonic Reduction:} As long as $\alpha\leq\alpha_3(\beta)\leq\alpha_2$, we have strong convexity over the line-segment connecting $\theta,\theta^+$. By the Taylor's expansion theorem, we have:
\begin{small}
\begin{align*}
&V_\alpha(\theta^+,w)   \\
=&V_\alpha(\theta,w)+\FPPN{V_\alpha}{\theta}^T(\theta^+-\theta)+    \\
&\frac{1}{2}(\theta^+-\theta)^T\FPPTN{V_\alpha}{\theta}(\theta^+-\theta)+R_V    \\
=&V_\alpha(\theta,w)-\FPPN{V_\alpha}{\theta}^T\FPPN{\bar{G}_\alpha}{\theta}^{-1}G_\alpha+ \\
&\frac{1}{2}(\theta^+-\theta)^T\FPPTN{V_\alpha}{\theta}(\theta^+-\theta)+R_V    \\
=&V_\alpha(\theta,w)-\FPPN{V_\alpha}{\theta}^T
(\FPPN{G_\alpha}{\theta}^{-1}-\Delta G_\alpha^{-1})G_\alpha+ \\
&\frac{1}{2}(\theta^+-\theta)^T\FPPTN{V_\alpha}{\theta}(\theta^+-\theta)+R_V    \\
=&\FPPN{V_\alpha}{\theta}^T\Delta G_\alpha^{-1}G_\alpha+
\frac{1}{2}(\theta^+-\theta)^T\FPPTN{V_\alpha}{\theta}(\theta^+-\theta)+R_V    \\
=&\FPPN{V_\alpha}{\theta}^T\Delta G_\alpha^{-1}G_\alpha+
\frac{1}{2}G_\alpha^T\FPPN{\bar{G}_\alpha}{\theta}^{-T}\FPPTN{V_\alpha}{\theta}
\FPPN{\bar{G}_\alpha}{\theta}^{-1}G_\alpha+R_V    \\
=&\FPPN{V_\alpha}{\theta}^T\Delta G_\alpha^{-1}G_\alpha+
\frac{1}{2}G_\alpha^T(\FPPN{G_\alpha}{\theta}^{-1}-\Delta G_\alpha^{-1})^T   \\
&\FPPTN{V_\alpha}{\theta}
(\FPPN{G_\alpha}{\theta}^{-1}-\Delta G_\alpha^{-1})G_\alpha+R_V    \\
=&\FPPN{V_\alpha}{\theta}^T\Delta G_\alpha^{-1}G_\alpha+
\frac{1}{2}G_\alpha^T\FPPN{G_\alpha}{\theta}^{-T}
\FPPTN{V_\alpha}{\theta}\FPPN{G_\alpha}{\theta}^{-1}G_\alpha+R_V+    \\
&\frac{1}{2}G_\alpha^T\Delta G_\alpha^{-T}
\FPPTN{V_\alpha}{\theta}\Delta G_\alpha^{-1}G_\alpha-
G_\alpha^T\Delta G_\alpha^{-T}
\FPPTN{V_\alpha}{\theta}\FPPN{G_\alpha}{\theta}^{-1}G_\alpha    \\
=&\FPPN{V_\alpha}{\theta}^T\Delta G_\alpha^{-1}G_\alpha+\frac{1}{2}V_\alpha(\theta,w)+R_V+    \\
&\frac{1}{2}G_\alpha^T\FPPN{G_\alpha}{\theta}^{-1}
\sum_i\FPPTN{[G_\alpha]_i}{\theta}[G_\alpha]_i\FPPN{G_\alpha}{\theta}^{-1}G_\alpha+    \\
&\frac{1}{2}G_\alpha^T\Delta G_\alpha^{-T}
\FPPTN{V_\alpha}{\theta}\Delta G_\alpha^{-1}G_\alpha-
G_\alpha^T\Delta G_\alpha^{-T}
\FPPTN{V_\alpha}{\theta}\FPPN{G_\alpha}{\theta}^{-1}G_\alpha.
\end{align*}
\end{small}
The above equation contains 5 terms other than $V_\alpha(\theta,w)/2$, we show that these terms are all $\mathcal{O}(\alpha)V_\alpha(\theta,w)$. The first term is:
\begin{align*}
&\|\FPPN{V_\alpha}{\theta}^T\Delta G_\alpha^{-1}G_\alpha\|
\leq\|G_\alpha^T\FPPN{G_\alpha}{\theta}\Delta G_\alpha^{-1}G_\alpha\|    \\
\leq&\|G_\alpha\|^2M_{\Delta G}(\frac{\sigma_X}{2\alpha}-\bar{\sigma}_B)^{-1}
=\mathcal{O}(\alpha)V_\alpha(\theta,w).
\end{align*}
The second term is:
\begin{align*}
\|R_V\|\leq&\frac{1}{6}(\frac{1}{\alpha^2}M_{AA}^3+\frac{1}{\alpha}M_{AB}^3+M_{BB}^3)   \\
&(\frac{\sigma_X}{2\alpha}-\bar{\sigma}_B)^{-3}\sqrt{M_V}V_\alpha(\theta,w)
=\mathcal{O}(\alpha)V_\alpha(\theta,w).
\end{align*}
The third term is:
\begin{small}
\begin{align*}
&\|\frac{1}{2}(\theta^+-\theta)^T\sum_i\FPPTN{[G_\alpha]_i}{\theta}[G_\alpha]_i(\theta^+-\theta)\| \\
\leq&\frac{1}{2}(\frac{1}{\alpha}M_A^2+M_B^2)\|\theta^+-\theta\|^2\|G_\alpha\|  \\
=&\frac{1}{2}(\frac{1}{\alpha}M_A^2+M_B^2)(\frac{\sigma_X}{2\alpha}-\bar{\sigma}_B)^{-2}
\sqrt{M_V}V_\alpha(\theta,w)
=\mathcal{O}(\alpha)V_\alpha(\theta,w).
\end{align*}
\end{small}
The forth term is:
\begin{align*}
&\|\frac{1}{2}G_\alpha^T\Delta G_\alpha^{-T}\FPPTN{V_\alpha}{\theta}\Delta G_\alpha^{-1}G_\alpha\|    \\
\leq&\|\frac{1}{2}G_\alpha^T\Delta G_\alpha^{-T}
\FPPN{G_\alpha}{\theta}^T\FPPN{G_\alpha}{\theta}\Delta G_\alpha^{-1}G_\alpha\|+   \\
&\|\frac{1}{2}G_\alpha^T\Delta G_\alpha^{-T}
\sum_i\FPPTN{[G_\alpha]_i}{\theta}[G_\alpha]_i\Delta G_\alpha^{-1}G_\alpha\|  \\
\leq&\frac{1}{2}M_{\Delta G}^2(\frac{\sigma_X}{2\alpha}-\bar{\sigma}_B)^{-2}V_\alpha(\theta,w)+   \\
&\frac{1}{2}(\frac{1}{\alpha}M_A^2+M_B^2)
M_{\Delta G}^2(\frac{\sigma_X}{2\alpha}-\bar{\sigma}_B)^{-2}(\frac{\sigma_X}{2\alpha}-\sigma_B)^{-2}    \\
&\sqrt{M_V}V_\alpha(\theta,w)=\mathcal{O}(\alpha)V_\alpha(\theta,w).
\end{align*}
The fifth term is:
\begin{align*}
&\|G_\alpha^T\Delta G_\alpha^{-T}\FPPTN{V_\alpha}{\theta}\FPPN{G_\alpha}{\theta}^{-1}G_\alpha\|  \\
\leq&\|G_\alpha^T\Delta G_\alpha^{-T}\FPPN{G_\alpha}{\theta}^TG_\alpha\|+    \\
&\|G_\alpha^T\Delta G_\alpha^{-T}
\sum_i\FPPTN{[G_\alpha]_i}{\theta}[G_\alpha]_i\FPPN{G_\alpha}{\theta}^{-1}G_\alpha\|    \\
\leq&M_{\Delta G}(\frac{\sigma_X}{2\alpha}-\bar{\sigma}_B)^{-1}V_\alpha(\theta,w)+  \\
&\frac{1}{2}(\frac{1}{\alpha}M_A^2+M_B^2)
M_{\Delta G}(\frac{\sigma_X}{2\alpha}-\bar{\sigma}_B)^{-1}(\frac{\sigma_X}{2\alpha}-\sigma_B)^{-2}    \\
&\sqrt{M_V}V_\alpha(\theta,w)=\mathcal{O}(\alpha)V_\alpha(\theta,w).
\end{align*}
The rest of the proof is the same as \prettyref{Lem:discreteInvariance}. Note that \prettyref{Thm:DConv} holds in the zeroth-order case without \prettyref{ass:GwFullrank} because we have not used the condition $\sigma_{min}(\FPPN{G}{w}\FPPN{G}{w}^T)\geq\sigma_w>0$.
\end{proof}

%% file: convergence.tex
\subsection{First-Order PGM Convergence}
The convergence of the first-order PGM is a trivial result of \prettyref{Thm:DConv} and Taylor's expansion theorem.

\begin{proof}[\prettyref{Thm:PGM}]
If we set $\alpha_4$ as in \prettyref{Thm:DConv}, then each $\theta^k$ satisfies $G(\theta^k,w^k)=0$ and $K(\theta)$ can be expressed as a function of $w$, denoted as $K(w)$. Due to \prettyref{ass:XSmooth} and \prettyref{ass:VSmooth}, $K(w)$ is differentiable and with gradient being:
\begin{align*}
\FPPN{K}{w}\triangleq
-\FPPN{G}{w}^T\FPPN{G_\alpha}{\theta}^{-T}\FPPN{K}{\theta}.
\end{align*}
From the optimality condition of \prettyref{eq:NMDPL}, we know that $\angle(\FPPN{K}{w},\Delta w)\to\pi$ and $\Delta w\to0$ as $\gamma\to0$, so $\{K(\theta^k)\}$ is monotonically decreasing convergence is guaranteed.
\end{proof}

\subsection{Zeroth-Order PGM Convergence}
The convergence proof of the zeroth-order PGM is a slight variant of the first-order \prettyref{Thm:PGM}. The key is to show that a descendent direction is always achieved with small enough $\alpha$.

\begin{proof}[\prettyref{Thm:ZOPGM}]
If we set $\alpha_5\leq\alpha_4$ as in \prettyref{Thm:DConv}, then each $\theta^k$ satisfies $G(\theta^k,w^k)=0$ and $K(\theta)$ can be expressed as a function of $w$, denoted as $K(w)$. Now we define the modified gradient:
\begin{align*}
\FPPN{\bar{K}}{w}\triangleq
-\FPPN{G}{w}^T\FPPN{\bar{G}_\alpha}{\theta}^{-T}\FPPN{K}{\theta}.
\end{align*}
From the optimality condition of \prettyref{eq:ZONMDPL}, we know that $\angle(\FPPN{\bar{K}}{w},\Delta w)\to\pi$ and $\Delta w\to0$ as $\gamma\to0$. Next, we check whether $\FPPN{\bar{K}}{w}$ is also a descendent direction of $K(\theta^k)$:
\begin{small}
\begin{align*}
&\|\FPPN{\bar{K}}{w}^T\FPPN{K}{w}\|  \\
=&\|\FPPN{K}{\theta}^T\FPPN{\bar{G}_\alpha}{\theta}^{-1}\FPPN{G}{w}
\FPPN{G}{w}^T\FPPN{G_\alpha}{\theta}^{-T}\FPPN{K}{\theta}\|  \\
=&\|\FPPN{K}{\theta}^T\FPPN{G_\alpha}{\theta}^{-1}\FPPN{G}{w}
\FPPN{G}{w}^T\FPPN{G_\alpha}{\theta}^{-T}\FPPN{K}{\theta}\|-  \\
&\|\FPPN{K}{\theta}^T\Delta G_\alpha^{-1}\FPPN{G}{w}
\FPPN{G}{w}^T\FPPN{G_\alpha}{\theta}^{-T}\FPPN{K}{\theta}\| \\
\geq&\|\FPPN{K}{\theta}\|^2\sigma_w(\frac{1}{\alpha}M_A^1+M_B^1)^{-2}-
M_w\|\FPPN{K}{\theta}\|^2(\frac{1}{\alpha}M_A^1+M_B^1)^{-1}  \\
&M_{\Delta G}(\frac{\sigma_X}{2\alpha}-\bar{\sigma}_B)^{-1}(\frac{\sigma_X}{2\alpha}-\sigma_B)^{-1}\BECAUSE{\prettyref{Lem:basicZO},\prettyref{ass:GwFullrank}} \\
=&\mathcal{O}(\alpha^2)\|\FPPN{K}{\theta}\|^2-\mathcal{O}(\alpha^3)\|\FPPN{K}{\theta}\|^2
\geq\mathcal{O}(\alpha^2)\frac{1}{2}\|\FPPN{K}{\theta}\|^2.
\end{align*}
\end{small}
The last inequality must hold by choosing small enough $\alpha_5\leq\alpha_4$. The proof is complete.
\end{proof}